\documentclass[10pt,twocolumn,letterpaper]{article}

\usepackage{iccv}
\usepackage{times}
\usepackage{epsfig}
\usepackage{graphicx}
\usepackage{amsmath}
\usepackage{amssymb}

\usepackage{amssymb,amsmath,amsthm}

\newtheorem{theorem}{Theorem}

\usepackage{mathtools} 

\usepackage{amsthm}
\usepackage{amsmath}
\theoremstyle{definition}
\newtheorem{definition}{Definition}[section]
\DeclareMathOperator*{\argmin}{argmin}
\usepackage{graphicx}
\usepackage[labelformat=simple]{subcaption}

\usepackage{algorithm,algorithmic,amsmath}
\usepackage[export]{adjustbox}
\usepackage{caption}
\usepackage{subcaption}
\usepackage{tabularx}
\usepackage{wrapfig}
\usepackage{dblfloatfix}
\usepackage{amsmath,amsfonts,amssymb}
\usepackage{bm}
\usepackage{comment}

\usepackage{titletoc}

\usepackage[pagebackref=true,breaklinks=true,colorlinks,bookmarks=false]{hyperref}

\iccvfinalcopy 

\def\iccvPaperID{7033} 

\ificcvfinal\fi

\usepackage[accsupp]{axessibility}  

\begin{document}

\title{Explaining Local, Global, And Higher-Order Interactions In Deep Learning}


\author{Samuel Lerman \qquad Charles Venuto \qquad Henry Kautz \qquad Chenliang Xu\\
University of Rochester\\
{\tt\small \{slerman@ur.,charles.venuto@chet.,kautz@cs.,chenliang.xu@\}rochester.edu} 
}

\maketitle
\ificcvfinal\thispagestyle{empty}\fi

\begin{abstract}
   We present a simple yet highly generalizable method for explaining interacting parts within a neural network’s reasoning process. First, we design an algorithm based on cross derivatives for computing statistical interaction effects between individual features, which is generalized to both 2-way and higher-order (3-way or more) interactions. We present results side by side with a weight-based attribution technique, corroborating that cross derivatives are a superior metric for both 2-way and higher-order interaction detection. Moreover, we extend the use of cross derivatives as an explanatory device in neural networks to the computer vision setting by expanding Grad-CAM, a popular gradient-based explanatory tool for CNNs, to the higher order. While Grad-CAM can only explain the importance of individual objects in images, our method, which we call Taylor-CAM, can explain a neural network’s relational reasoning across multiple objects. We show the success of our explanations both qualitatively and quantitatively, including with a user study. We will release all code as a tool package to facilitate explainable deep learning.
\end{abstract}

\section{Introduction}

\begin{figure}
\centering
\includegraphics[width=\linewidth]{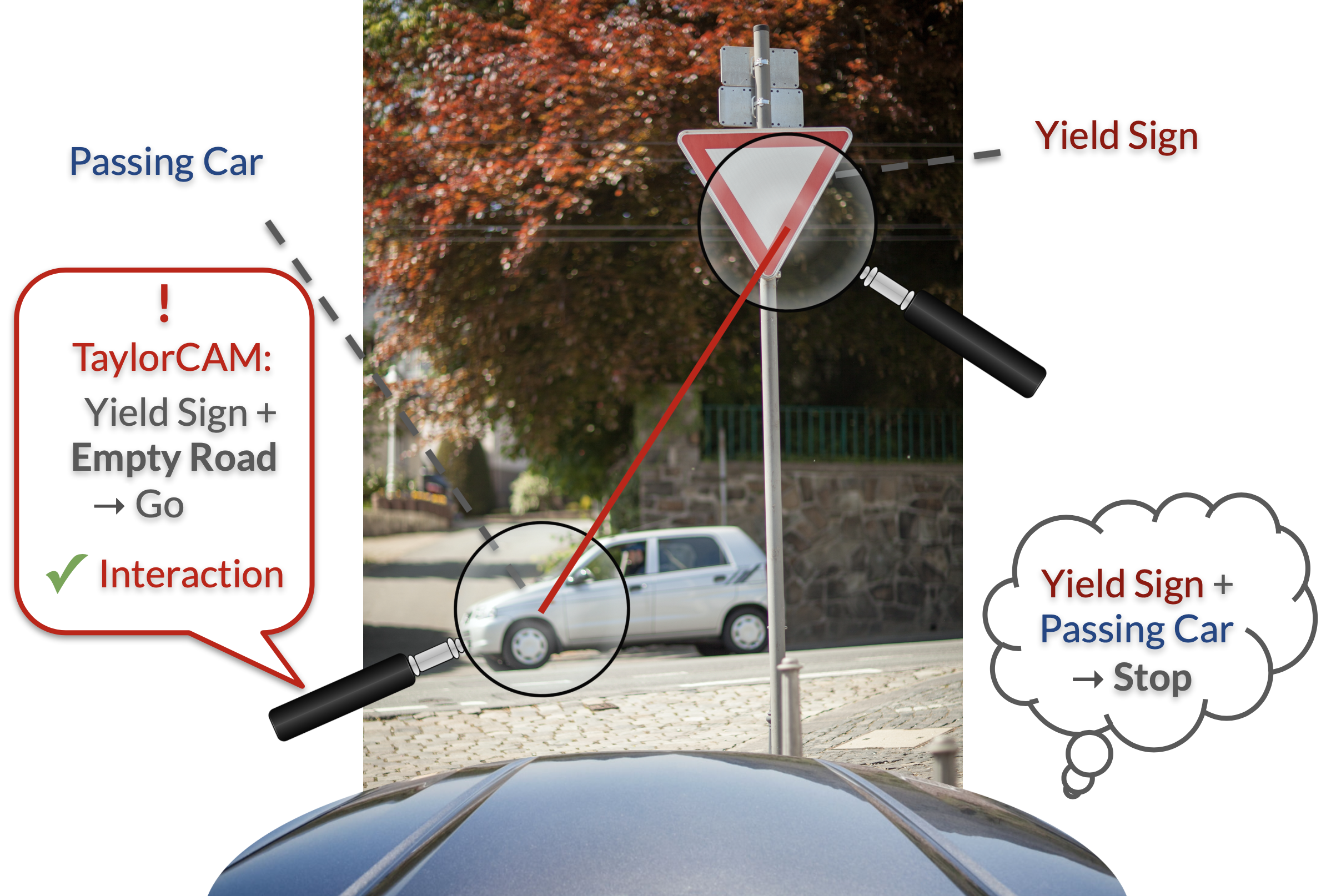}
\caption{An automated driver decides whether to ``stop'' or ``go.'' Here, the decision cannot be explained by individual factors alone, but by the interaction between the yield sign and the passing car. Taylor-CAM identifies interactions by considering how changing one object affects the significance of another, such as how changing a passing car into an empty road would change the meaning of the yield sign from ``stop'' to ``go.''}
\end{figure}

The universe is made up of myriad interacting parts. To truly understand complex systems and processes, it is not enough to view their functions as an amalgamation of independent contributors. Rather, they are a complex web of inter-operating influences. For much of the past, explainable deep learning has concerned itself with identifying important features, feature vectors, and isolated concepts. However, in the real world, humans intuitively understand that decisions are consequences of complex relations, not merely extrapolated from rankings of singular phenomena.

To explain an agent’s strategy in any task — be it computer vision, natural language processing, physics, biomedicine, reinforcement learning, autonomous driving, economics, or future forecasting — is imprecise without an interactional approach. 
However, interactional strategies are not always summarizable by heatmaps \cite{gradcampp, gradcam, heat4, heat5, cam} or ordered rankings \cite{inputgradients3, inputgradients1, inputgradients2}; and they often require an understanding of many dependencies --- complex dependencies, such as those between higher-level concepts (\textit{e.g.} vector representations in deep neural networks \cite{relationalreasoning, rrnn, RN, rdrl}) --- not just single-dimensional features as typically explored in the statistical interaction effects literature \cite{eberle, janizek, tsang2, NID}.
In light of all of this, we propose a number of contributions towards explaining interactions in deep learning:

\textbf{T-NID}, an algorithm for statistical interaction effects that outperforms recent state-of-the-art baselines with both pairwise and higher-order interactions. Interaction effects are a fundamental notion in statistics \cite{ie1}. We make this computation tractable by translating local interaction effects into global interaction effects via representative samples and employing a simple subsampling heuristic. 

\textbf{Taylor-CAM}, an explanatory tool that extends Grad-CAM \cite{gradcam}, which assigns importances to feature vectors based on input gradients, by generalizing it to the 2-way and higher-order setting using the same formalism of interaction effects as for T-NID. This method is demonstrated on multi-object detection and relational reasoning in visual question-answering (VQA).

\textbf{Visualizations} of Taylor-CAM's explanations that enable a human cohort to reverse engineer questions in relational VQA without knowing the answers and interpret relational reasoning better than with existing explanatory tools like Grad-CAM and GLIDER \cite{tsang2} from just a convolutional neural network's (CNN) feature maps. 

\section{Related Work}

\textbf{In Deep Learning} \quad
Recently, there have been several attempts to compute statistical interactions with deep learning. Neural Interaction Detection (NID) \cite{NID} used neural network weights to interpret interactions, observing that interactions occur at nonlinear activations in the first hidden layer of an MLP. Like our approach T-NID, \cite{bayes} used gradient information to compute statistical interaction effects. However, they relied on Bayesian neural networks, required averaging a high number of Hessians, and only computed global interaction effects, not focusing on local or higher-order interactions. \cite{eberle} used cross derivatives between single features to explain interactions in deep similarity models, whereas we use an adaptation of Grad-CAM to demonstrate explainability in a more general computer vision setting. \cite{autoint} relied on self attention \cite{attention} to compute a measure analogous to non-emergent interaction effects and apply this to an analysis in the biomedical domain. 
Higher-order interactions have been considered throughout biomedicine, particularly for understanding gene interactions \cite{gene4, caruana, gene2, gene1, gene3}.

Cui et al.~\cite{bayes} applied their approach to a toy MNIST dataset consisting of a fixed set of feature vectors such that they could compute global interaction effects, but they mapped those feature vectors to single neurons and computed standard interaction effects between those mapped neurons. The limitation of this approach is that it cannot be used to explain local phenomena, which is traditionally what is of interest in computer vision, NLP, and other areas where multidimensional feature vectors are used.

\cite{janizek} and \cite{tsang2}, like our substitution of ReLU with GELU, substitute ReLU with Softplus in order to induce differentiability. The latter, like our work, translate local interaction effects to global interaction effects by aggregating across representative samples. While they use a random batch, we use a small subset of common aggregates. 
While our Taylor-CAM formulation is expressly adapted from Grad-CAM for intuitively explaining feature vectors in CNNs, \cite{janizek} derive their formulation from integrated gradients and \cite{tsang2, tsang} directly use cross partials.

\textbf{Individual Importances} \quad
\cite{inputgradients3, inputgradients1, inputgradients2} used input gradients to explain the reasoning of a neural network. \cite{cam} did so with class activation maps. Grad-CAM \cite{gradcam} and Grad-CAM++ \cite{gradcampp} combined both approaches to localize important feature vectors in computer vision with class activation maps and gradients, visualized by heatmaps. Similar to us, \cite{montavon} used Taylor decomposition to explain neural network decisions, but only for main effects, not interactions. 

\textbf{Relational Reasoning} \quad
We also connect interaction effects with relational reasoning, which has received increased attention in deep learning \cite{relationalreasoning, rrnn, RN, rdrl}, and use Taylor-CAM to interpret the reasoning process of Relation Networks \cite{RN}. While most past works have mainly focused on explaining individual factors of a neural network's predictions, the weights in multi-head dot product attention \cite{attention} could be interpreted as relational explanations for neural networks that include MHDPA in their architecture \cite{autoint}. In contrast, Taylor-CAM is architecture agnostic and can explain decisions unique to each output dimension directly from gradient information.

Unlike other works, we expressly derived Taylor-CAM for the purpose of explaining interactions between higher level representations, such as feature maps from a CNN, which standardly represent objects in computer vision (rather than using raw RGB pixels). 
As Grad-CAM is built on projected feature vectors in addition to gradients, so is our higher-order extension w.r.t. cross derivatives to explain interactions rather than isolated phenomena. 

\section{Statistical Interaction Effects}
\label{statiessec}
We define statistical interaction effect analogous to \cite{formaldef}:

\theoremstyle{definition}
\begin{definition}{\textbf{Interaction Effect}}
\label{iedef}
An interaction effect $\bm{\mathit{IE}}_{1,...,\ell}$ between variables $\bm{x}_1, ..., \bm{x}_\ell \in \bm{x}$ on a function $F(\bm{x})$ with inputs $\bm{x}$ is measured as:
\begin{equation}
\bm{\mathit{IE}}_{1, ..., \ell} = \frac{\partial^\ell F(\bm{x})}{\partial \bm{x}_1 \dotsm \partial \bm{x}_\ell}.
\end{equation}
\end{definition}
%
In plain English, an interaction effect is how much the meaning of one variable changes for a unit change in another variable. This change is reflected by the cross partial derivative. ``Change'' is an intuitive measure for interaction. From the earlier example, given a representation of a yield sign and an oncoming car, \textit{changing} the representation of the oncoming car into a representation of an empty road also changes the meaning of the yield sign from ``stop'' to ``go.'' For a more formal example, consider $F(\bm{x}) = \bm{x}_1sin(\bm{x}_2) + cos(\bm{x}_3)$. $F$ consists of an interaction between $\bm{x}_1$ and $\bm{x}_2$ for some $\bm{x}$ since $\partial^2 F(\bm{x}) / (\partial \bm{x}_1 \partial \bm{x}_2)$ is nonzero. However, $\bm{x}_3$ does not belong to an interaction since any cross derivative w.r.t. $\bm{x}_3$ is zero. 


\textbf{Adapt to Neural Networks} \quad Substituting $F$ with a trained neural network, we can compute the local interaction effects for a datapoint up to order $\ell$ as long as the neural network $F$ is $\ell$-times differentiable. In classification, softmax ensures this to be the case. In regression, we substitute ReLUs with Gaussian-error rectified linear units (GELUs), which have been shown comparable in performance \cite{gelu}. Otherwise, Definition \ref{iedef} affords the computation of interaction effects for arbitrary neural network
architectures.

\textbf{Translate Local Effects to Global Effects} \quad 
Often
in statistics, there is greater interest in computing global interaction effects, statistics that generalize across all datapoints. Similarly, this need may be found in analyzing scene graphs, object co-currency, and contextual information \cite{co-oc2, co-oc1, co-oc3}. In tandem with our work, \cite{bayes} converted local pairwise interaction effects to global pairwise interaction effects by averaging a set of representative samples retrieved via k-means clustering, in effect dividing the dataset by Euclidean distance and computing the global average from the centroids. We will similarly average representative local interaction effects in order to compute a global summary, but we will use a simpler and more efficient technique. In our case, efficiency is of more concern because computing higher-order interaction effects requires the computation of higher-order derivatives, which for many samples can become intractable. 

To translate local interaction effects into global 
ones 
at any order, we sample representative samples that have a wide range over the dataset and that are potentially meaningful. We choose the samples that are closest to a subset of common aggregates, including mean, median, min, max, and mode. As well as a random sample for good measure. Likewise, we used L2 distance to measure closeness. In addition, we considered different ways to aggregate the interaction effects of these samples. Again, namely mean, median, min, max, or mode. We ran a wide sweep of the complete power set of these potential samples and aggregates to find which combination performed best on a wide array of synthetic datasets distinct from those we trained on selected from prior works \cite{ieprior3, ieprior2, ieprior1, NID}, chosen to test for various types of interactions. Results of this power sweep are reported in the \textit{Appendix}. We ended up using the mean interaction effect of the samples closest to the mean, minimum, and mode of all samples, as well as a random sample.

\textbf{Improve Efficiency} \quad Another heuristic for efficiency that we employed was subsampling the interactions that would be computed. Naturally, testing for every combination up to order $\ell$ would be very expensive. Every double, every triple, every quadruple, etc. --- the problem grows combinatorially. We were able to mitigate this to a degree by taking advantage of the property of statistical interaction effects that \textit{an $\ell$-way interaction can only exist if all its corresponding ($\ell$ - 1)-interactions exist} \cite{ieprior1}. In turn, we were able to reduce the search space by only selecting non-redundant combinations of the $k$ interactions from the previous order whose interaction effects were highest, beginning with using every combination up to order $o$ and then subsampling the top $k$ for every order thereafter. 

Our complete algorithm, which we call Taylor-Neural Interaction Detection (T-NID) due to the higher-order derivatives, is described in pseudocode in the \textit{Appendix}.

Finally, we need to make a point about the sign of the resulting cross partial derivatives. A positive value indicates change in the positive direction; negative, negative. Since in regression we are interested in the overall effect of an interaction and are agnostic to the direction, we take the squared value of the cross-partial as our measure of interaction effect. In contrast, for classification, we use the sign --- positive or negative --- corresponding to the class of interest. 
And for multi-class classification, we take $F$ to be the network corresponding to the class output of interest, 
and use its squared cross partial derivatives.


\section{Taylor-CAM}

To this point, we have generalized our computation of interaction effects to the local, global, and higher-order setting, but we have not yet considered the case where features are multidimensional, as is the case in higher-level deep neural network representations.

Explaining the influence of feature vectors is common in computer vision and interpreting CNNs. However, we have illustrated with multiple examples why a precise explanation of a model's decisions requires an explanation of its interacting components, not just 
singular entities.

\subsection{Intuition}
\label{intuitionsec}

For arbitrary objects in the computer vision setting, a cross derivative alone is not sufficient. Besides the obvious reason that such objects are not represented by singular features but by multidimensional feature vectors learned by a CNN, it is also because fundamentally a cross derivative measures changes of changes. More formally, a cross derivative $\frac{\partial^2 F}{\partial x \partial y}$ measures the effect of a unit change of $x$ on the effect on $F$ of a unit change of $y$. When reasoning about visual relations, it is convenient to think of dependencies between objects that inform a decision, such as the dependency between a yield sign and a passing car in informing an automated driver's decision to ``stop'' or ``go.'' \textit{Changing} the passing car into another object, such as merely an empty road, would on its own change the neural network's interpretation of the yield sign from meaning ``stop'' to meaning ``go,'' even while keeping the yield sign fixed and unchanged --- yet a cross derivative only measures the effect of changing both. To account for this, instead of naively using cross derivatives, we measure how much changing one object would change the \textit{importance} of another object to a neural network's decision, \textit{e.g.}, how changing the yield sign into a speed limit sign would change the passing car's importance or how changing the passing car into a gush of leaves would change the yield sign's importance with regards to the decision of whether to ``stop'' or ``go'' --- even when not necessarily both are changed.

Given car $C$, yield sign $Y$, and binary decision ``go'' $G$, this intuition may be summarized mathematically as:

\begin{equation}
\label{intuition}
S_{Y,C} = \partial \mathit{IMP}(Y, G)\big/\partial C \text{,}
\end{equation}
where $\mathit{S}_{\mathit{Y},\mathit{C}}$ represents the interaction salience between the yield sign and passing car, and $\mathit{IMP}(\mathit{Y}, \mathit{G})$ represents the importance of the yield sign to the neural network's decision to go or stop. Fortunately, the importance of individual objects in computer vision is the characteristic problem of the explanatory tool Grad-CAM \cite{gradcampp, gradcam, cam}, which we use to derive our method. We use the term \textit{interaction salience} due to 
deviation from interaction effects in 
Definition \ref{iedef}. 

\subsection{Methodology}

Suppose we have an $\ell$-times differentiable function $F: \mathbb{R}^{n, d} \to \mathbb{ R}$, which will stand for our neural network, where $\ell \geq 2$. $F$ takes in matrix $\bm{x}$ consisting 
of $n$ feature vectors $\bm{x}_1, ..., \bm{x}_n \in \mathbb{R}^{d}$ of dimension $d$. So $\bm{x}_1, ..., \bm{x}_n$ are just feature vectors produced by a CNN and each one is associated with an image region. $F$ is the portion of the network downstream of 
these
feature vectors.


\textbf{Quantify Importance} \quad To fill $\mathit{IMP}$ in Equation \ref{intuition}, we turn to class activation maps (CAMs) \cite{cam}. However, as observed by the solution of \cite{gradcam}, to find out how a class activation map increases the class's likelihood, we would like to know how its features contribute to the output, which we can do with their gradients. We can estimate the global effect by summing the gradient of each feature vector $\bm{x}_k$ and weighing the sum to each CAM. This amounts exactly to Grad-CAM \cite{gradcam}:
\begin{equation}
\label{imp}
\begin{split}
\mathit{IMP}(\bm{x}_i, F(\bm{x})) &= \mathit{GradCAM}(\bm{x}_i, F(\bm{x})) \\&= \sum\limits_{p} \bm{x}_{ip} \sum\limits_{k} \frac{\partial F(\bm{x})}{\partial \bm{x}_{kp}}
\end{split}.
\end{equation}
\textbf{Generalize Grad-CAM to Compute Interactions} \quad Now that we have the importance of a feature vector (via essentially Grad-CAM), we can formulate $\bm{S}_{ij}$, the interaction salience between feature vectors $\bm{x}_i$ and $\bm{x}_j$, by substituting Equation \ref{imp} into \ref{intuition} and summing the dimensions as follows:
\begin{equation}
\label{sub}
\bm{S}_{ij} = \sum\limits_{m} \partial \left[\sum\limits_{p} \bm{x}_{ip} \sum\limits_{k} \frac{\partial F(\bm{x})}{\partial \bm{x}_{kp}}\right] \bigg/ \partial \bm{x}_{jm}.
\end{equation}
\textbf{Merge with Statistical Interaction Effects} \quad Finally, we bring this to an easy-to-compute form by realizing that the partial derivative in the denominator $\partial \bm{x}_j$ can be computed together with the partial derivative in the numerator. We also square the salience because a change of importance in either direction would be significant. We note that the following is a generalization of Grad-CAM that reduces elegantly to a modified interaction effects Definition \ref{iedef}:
\begin{equation}
\label{hessiancam}
\begin{split}
\bm{S}_{ij}^2 &= \Bigg(\sum\limits_{m} \sum\limits_{p} \bm{x}_{ip} \sum\limits_{k} \frac{\partial^2 F(\bm{x})}{\partial \bm{x}_{kp} \partial \bm{x}_{jm}}\Bigg)^2 \\&= \Bigg(\sum\limits_{m, p, k} \bm{x}_{ip}\bm{\mathit{IE}}_{kp,jm}\Bigg)^2
\end{split}.
\end{equation}
In tests, we found setting $k = i$ in Equations \ref{imp} - \ref{hessiancam} without the global sum over $k$ to perform just as well and often better, perhaps because the local gradients in Equation \ref{imp} more precisely correspond to features. We call Equation \ref{hessiancam} Hessian-CAM. Hessian-CAM may be further differentiated with respect to a cross partial $\partial \bm{x}_q$ to get a $3$-way interaction salience, and that can be further differentiated up to any order $\ell$. Thus, we name this Taylor-CAM, a higher-order generalization of Grad-CAM, where Grad-CAM (or a close variant) is the special case $\ell = 1$ and Hessian-CAM is the special case $\ell = 2$. 

Note that interaction saliences are conditional. The interaction salience of feature $\bm{x}_i$ on feature $\bm{x}_j$ is not necessarily the same as that of $\bm{x}_j$ on $\bm{x}_i$. Interaction salience $\bm{S}_{ij}$ represents the influence of $\bm{x}_i$ on the importance of $\bm{x}_j$. Interaction salience $\bm{S}_{ijk...}$ represents the influence of $\bm{x}_i$ on the interaction salience of interaction $\bm{x}_j, \bm{x}_k, ...$. To address this, we sum the mutual pairs, \textit{e.g.}, $\bm{S}_{ij} + \bm{S}_{ji}$, although we note that we did so only to make the presentation clearer and not because it is required. For many interpretation tasks, understanding that the meaning of the yield sign depends on the car, but the meaning of the car does not depend on the yield sign is crucial to getting the most precise understanding. Computing the mutual pairs does not require re-computation of any derivatives, and can be achieved easily by permuting the resulting interaction saliences and summing them. Lastly, we zero out the diagonals and redundant grid cells of the resulting interaction saliences to only consider interactions between non-redundant feature vectors.
 
 
 \subsection{Limitations}


One limitation of Taylor-CAM, much like Grad-CAM, is that ``importance'' is based on contribution to the output,
so if two different objects have the same contribution to the output, 
then changing one into the other would be considered meaningless, 
and so the interactions might not be identified. 
Suppose we have the setup from Sort-Of-CLEVR \cite{clevr}, a relational reasoning task. Here, we have an image with an assortment of shapes of different colors and a relational question related to that image. An example of this limitation is when an agent is asked, ``What is the color of the circle furthest from the pink square?'' If the furthest circle is blue, and the second furthest is also blue, then changing the furthest into a square does not meaningfully impact the pink square’s contribution to the output, as determined by Grad-CAM, since the answer to the question would be unchanged (blue). Grad-CAM++ \cite{gradcampp} may hold an insight as to how to address this, via even-higher order derivatives. 

Another limitation is that ``change'' is 
measured locally, as derivatives do not account for non-local rates of change. This means that Taylor-CAM, like other deep learning explanatory tools, depends on 
local regions of representations. 

Lastly, of course, is the time complexity of computing higher-order derivatives. Higher-order differentiation has become increasingly more accessible with Taylor-mode autograd methods like JAX \cite{autograd1} and libraries like the new Pytorch functional autograd API \cite{autograd2}, yet remains a challenge as the order grows. For Hessian-CAM, we had no trouble computing 2nd-order derivatives of Relation Networks using Pytorch and CPU memory. None of our individual explanations required more than a few minutes to compute on a CPU, excluding neural network training. 

\section{Experiments}
\subsection{Statistical Interaction Effects} 
\begin{table*}
  \centering
            \caption{AUC scores for pairwise interaction effects. Top-1 scores are bolded.}
            \begin{tabular}{lcccccccc} 
            \hline & ANOVA & HierLasso & RuleFit & AG & NID \cite{NID} & NID MLP-M \cite{NID} & GLIDER \cite{tsang2} & T-NID  \\
            \hline $F_{1}(\mathbf{x})$ & $0.992$ & $\mathbf{1.00}$ & $0.754$ & $\mathbf{1}$ & $0.970$ & $0.995 \pm 4.4 \mathrm{e}-3$ & $0.973 \pm 0.01$ & $0.962 \pm 0.022$ \\ 
            $F_{2}(\mathbf{x})$ & $0.468$ & $0.636$ & $0.698$ & $0.88$ & $0.79$ & $0.85 \pm 3.9 \mathrm{e}-2$ & $0.84 \pm 0.097$ & $ \mathbf{0.885} \pm 0.039$ \\ 
            $F_{3}(\mathbf{x})$ & $0.657$ & $0.556$ & $0.815$ & $\mathbf{1}$ & $0.999$ & $\mathbf{1} \pm 0.0$ & $0.919 \pm 0.075$ & $0.999 \pm 0.001$ \\
            $F_{4}(\mathbf{x})$ & $0.563$ & $0.634$ & $0.689$ & $\mathbf{0.999}$ & $0.85$ & $0.996 \pm 4.7 \mathrm{e}-3$ & $0.951 \pm 0.073$ & $ 0.998 \pm 0.003$ \\ 
            $F_{5}(\mathbf{x})$ & $0.544$ & $0.625$ & $0.797$ & $0.67 $ & $\mathbf{1}$ & $\mathbf{1} \pm 0.0$ & $0.997 \pm 0.008$ & $0.991 \pm 0.016$ \\
            $F_{6}(\mathbf{x})$ & $0.780$ & $0.730$ & $0.811$ & $0.64$ & $\mathbf{0.98}$ & $0.70 \pm 4.8 \mathrm{e}-2$ & $0.767 \pm 0.033$ & $0.954 \pm 0.026$ \\
            $F_{7}(\mathbf{x})$ & $0.726$ & $0.571$ & $0.666$ & $0.81$ & $0.84$ & $0.82 \pm 2.2 \mathrm{e}-2$ & $0.751 \pm 0.207$ & $\mathbf{0.98} \pm 0.021$ \\ 
            $F_{8}(\mathbf{x})$ & $0.929$ & $0.958$ & $0.946$ & $0.937$ & $0.989$ & $0.989 \pm 4.5 \mathrm{e}-3$ & $0.998 \pm 0.005$ & $ \mathbf{1.0} \pm 0.0$ \\ 
            $F_{9}(\mathbf{x})$ & $0.783$ & $0.681$ & $0.584$ & $0.808$ & $0.83$ & $0.83 \pm 3.7 \mathrm{e}-2$ & $0.754 \pm 0.098$ & $\mathbf{0.98} \pm 0.023$ \\ 
            $F_{10}(\mathbf{x})$ & $0.765$ & $0.583$ & $0.876$ & $\mathbf{1}$ & $0.995$ & $0.99 \pm 2.1 \mathrm{e}-2$ & $0.974 \pm 0.027$ & $\mathbf{1.0} \pm 0.0$ \\
            \hline Average & $0.721$ & $0.698$ & $0.764$ & $0.87$ & ${0 . 9 2}$ & ${0 . 9 2} \pm 1.8 \mathrm{e}-2$ & $0.892 \pm 0.063$ & $\mathbf{0.975}\pm 0.015$ \\
            \hline
            \end{tabular}
            
  \label{2wayies}
            \end{table*}
            We evaluate T-NID’s ability to rank interactions on the suite of synthetic functions 
            proposed by \cite{ieprior3, ieprior2, ieprior1, NID}, 
            which were “designed to have a mixture of pairwise and higher-order interactions, with varying order, strength, nonlinearity, and overlap” \cite{NID}. 
            These are available to see in the \textit{Appendix} and in Table 1 of \cite{NID}. 

\textbf{Pairwise Interactions} \quad For pairwise interaction effects (see Table \ref{2wayies}), we report or reproduce the experiments of \cite{NID} verbatim, measuring AUC scores between predicted interaction rankings and ground truths. A pair $x_i, x_j$
is considered an interaction either by itself or when it is a subset of a higher-order interaction,
as in \cite{ieprior2, ieprior1}. Included for comparison are benchmarks from various statistical and machine learning methods \cite{ieprior1, ie2, tsang2, NID, ie1}. NID \cite{NID} uses an interpretation of the weights from a standard MLP to detect interactions, whereas NID + MLP-M uses an MLP with additional univariate networks summed at the output to discourage modeling of main effects and false spurious interactions. 
GLIDER \cite{tsang2} is a recent cross-partial method that induces higher-order differentiability with Softplus.

In contrast, T-NID uses only a standard MLP and GELU activations. GELU demonstrably performs better.
Unlike NID, 
we found 
no 
benefit from MLP-M or sparsity regularization.
Despite the simpler architecture, T-NID is immune to some of the deficits of NID and NID + MLP-M. 
T-NID is able to distinguish main effects and spurious interactions in $F_2$ and $F_4$, and
while NID + MLP-M modeled spurious main effects in the $\{8, 9, 10\}$ interaction of $F_6$ and GLIDER appears to struggle with this as well, T-NID recognizes it as an interaction. All around, T-NID performs on par or better than NID and GLIDER at computing pairwise statistical interaction effects on these synthetic tasks.
\begin{table*}
  \centering
  \caption{AUC scores for higher-order $n$-way interaction effects}
\begin{tabular}{lcccccc}
                    \hline 
                    &
                    \multicolumn{2}{c}{3-Way Interactions} &
                    \multicolumn{2}{c}{4-Way Interactions} &
                   \multicolumn{2}{c}{5-Way Interactions}  \\
                  \cline{2-3}\cline{4-5}\cline{6-7}
                    & NID \cite{NID} & T-NID  & NID \cite{NID} & T-NID  & NID \cite{NID}  & T-NID  \\
            \hline Average & $0.08 \pm 0.013$ & $\textbf{0.76} \pm 0.07$ & $0.75 \pm 0.13$ & $\textbf{0.78} \pm 0.11$ & $0.92 \pm 0.06$ & $\textbf{0.97} \pm 0.05$ \\
                    \hline
                    \end{tabular}
  \label{nwayies}
                    \end{table*}
                    
                    \textbf{Higher-Order Interactions} \quad For higher-order interactions, we do not report AUC scores against the full ground truth, as that would grow combinatorially more expensive with higher orders. Since NID also extracts interactions one order at a time, we compare the AUC scores of NID and T-NID one order at a time and use ground truths from the union of their discovered interactions. That way, they can be assessed relative to one another, albeit not universally. In addition to the results reported in Table \ref{nwayies}, we tested many variants of architectures and report results with NID + MLP-M in the \textit{Appendix}. In all cases, the relative results were largely the same, with T-NID achieving the highest scores, except less so at $4$-way interactions when equipped with its own main effects network (MLP-M). Since any-order NID tends to find supersets much better than subsets, at $3$-way interactions, NID misses nearly all present interactions, whereas T-NID fares relatively well. Along with recent works \cite{bayes}, we have shown that cross derivatives are a promising metric for interaction attribution in DNNs.
                    
\subsection{Object Detection}

\begin{figure}
\centering
\begin{subfigure}{\linewidth}
\centering
\begin{subfigure}{\linewidth}
\begin{subfigure}{0.32\textwidth}
\includegraphics[width=\textwidth, frame]{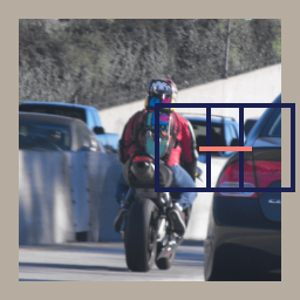}
\end{subfigure}
\hfill
\begin{subfigure}{0.32\textwidth}
\includegraphics[width=\textwidth, frame]{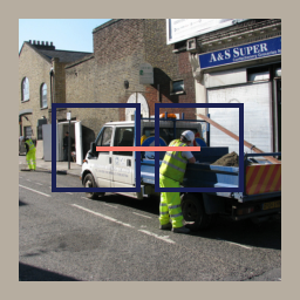}
\end{subfigure}
\hfill
\begin{subfigure}{0.32\textwidth}
\includegraphics[width=\textwidth, frame]{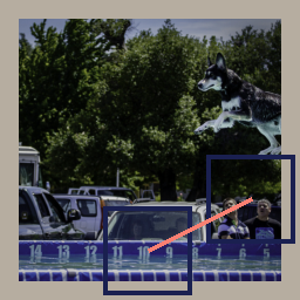}
\end{subfigure}
\caption{
Objects ``person'' and ``car'' are interacted to produce the output classification of whether both are present in the image in tandem.
}
\label{cocotask}
\end{subfigure}

\smallskip
\begin{subfigure}{\linewidth}
\begin{subfigure}{0.32\textwidth}
\includegraphics[width=\textwidth, frame]{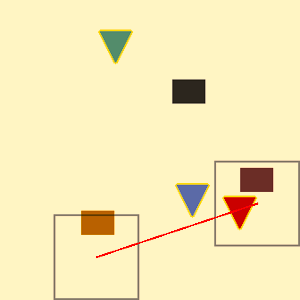}
\end{subfigure}
\hfill
\begin{subfigure}{0.32\textwidth}
\includegraphics[width=\textwidth, frame]{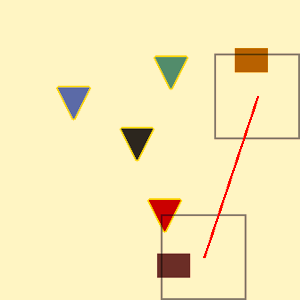}
\end{subfigure}
\hfill
\begin{subfigure}{0.32\textwidth}
\includegraphics[width=\textwidth, frame]{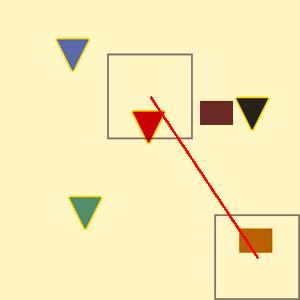}
\end{subfigure}
\end{subfigure}
\caption{Taylor-CAM interacts the yield sign (red triangle) with any present car (rectangle).}
\label{yieldgo}
\end{subfigure}
\vfill
\medskip
\begin{subfigure}{\linewidth}
\centering
\begin{subfigure}{0.49\textwidth}
\includegraphics[width=\textwidth, frame]{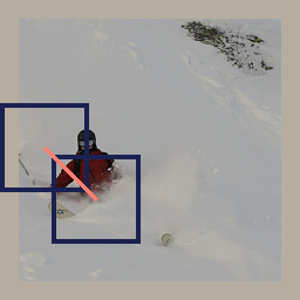}
\end{subfigure}
\hfill
\begin{subfigure}{.49\textwidth}
\centering
\includegraphics[width=\textwidth, frame]{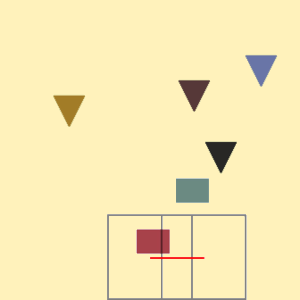}
\end{subfigure}
\caption{When no interactions present, Taylor-CAM's interactions intuitively are $0$ or occur primarily between adjacent regions as above.}
\label{selfinteracting}
\end{subfigure}
\caption{Top-1 bounding boxes generated by Taylor-CAM representing simple interactions in multi-object detection.}
\end{figure}

We ran two qualitative assessments of Taylor-CAM in multi-object detection. In both, the task was to identify whether a pair of objects were present in tandem. We tested the objects ``car'' and ``person'' in the COCO annotated-image dataset \cite{coco}, and we designed our own toy dataset consisting of cars (rectangles), signs (triangles), and a yield sign (red triangle) with labels ``go'' or ``stop.'' The COCO task suffered from model overfitting and lower test accuracy due to the limited pairwise data, but we still observed sensible explanations. 
Figure \ref{cocotask} shows such interactions assigned the highest interaction salience by Taylor-CAM. 


In the Yield-or-Go task, Taylor-CAM revealed two prediction strategies. The first is expected: the model interacts the yield sign (red triangle) with a car (rectangle), as seen in Figure \ref{yieldgo}, then predicts ``stop'' accordingly. In the second, the model interacts one car with all of the other cars. One would expect it to relate the car and the yield sign, but the model discovered that the problem can be solved by checking if (1) a car is present, and (2) a red car is not present. Since each object has a different color, (2) implies that a yield sign is present and thus to ``stop.'' 
Demystifying such reasoning strategies is a unique benefit of Taylor-CAM.

However, when the correct label is “go,” \textit{i.e.}, a car and yield sign are not present together, Taylor-CAM finds that the model rarely interacts anything, but rather either all interaction saliences are zero or objects interact with themselves (immediately adjacent regions) (Figure \ref{selfinteracting}). This self-interacting is an intuitive and convenient interpretation that Taylor-CAM provides in the lack of salient interactions. 

\subsection{Relational Reasoning}

\begin{figure}
\centering
\begin{subfigure}{\linewidth}
\centering
\begin{subfigure}{0.22\textwidth}
\includegraphics[width=\linewidth, frame]{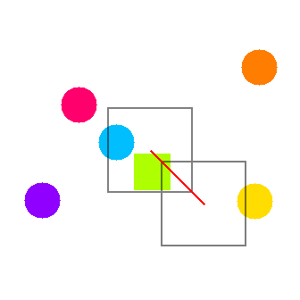}
\end{subfigure}
\hfill
\begin{subfigure}{0.22\textwidth}
\includegraphics[width=\linewidth, frame]{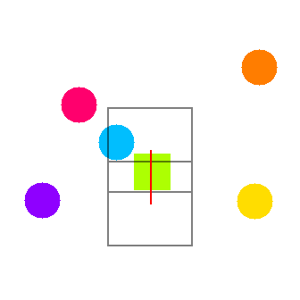}
\end{subfigure}
\hfill
\begin{subfigure}{0.22\textwidth}
\includegraphics[width=\linewidth, frame]{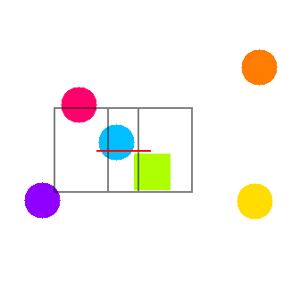}
\end{subfigure}
\hfill
\begin{subfigure}{0.22\textwidth}
\includegraphics[width=\linewidth, frame]{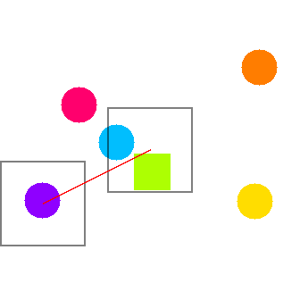}
\end{subfigure}%
\caption{Q: ``Which shape is closest to the green square?''}%
\medskip%
\begin{subfigure}{0.22\textwidth}
\includegraphics[width=\linewidth, frame]{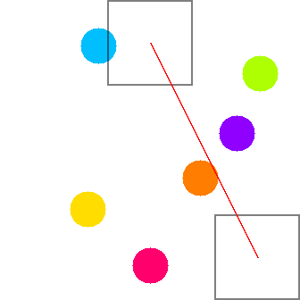}
\end{subfigure}
\hfill
\begin{subfigure}{0.22\textwidth}
\includegraphics[width=\linewidth, frame]{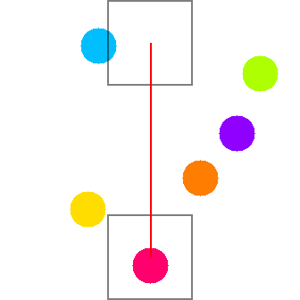}
\end{subfigure}
\hfill
\begin{subfigure}{0.22\textwidth}
\includegraphics[width=\linewidth, frame]{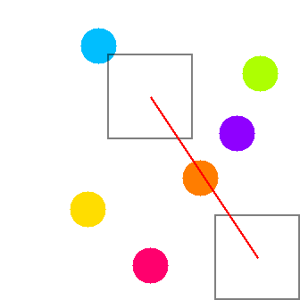}
\end{subfigure}
\hfill
\begin{subfigure}{0.22\textwidth}
\includegraphics[width=\linewidth, frame]{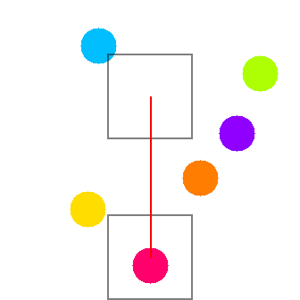}
\end{subfigure}%
\caption{Q: ``Which shape is furthest from the blue circle?''}%
\medskip%
\begin{subfigure}{0.22\textwidth}
\includegraphics[width=\linewidth, frame]{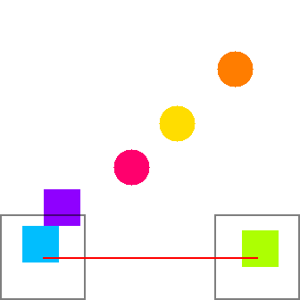}
\end{subfigure}
\hfill
\begin{subfigure}{0.22\textwidth}
\includegraphics[width=\linewidth, frame]{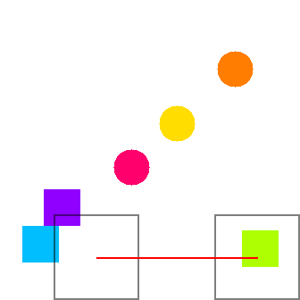}
\end{subfigure}
\hfill
\begin{subfigure}{0.22\textwidth}
\includegraphics[width=\linewidth, frame]{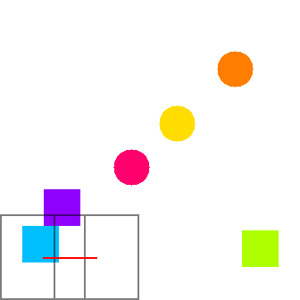}
\end{subfigure}
\hfill
\begin{subfigure}{0.22\textwidth}
\includegraphics[width=\linewidth, frame]{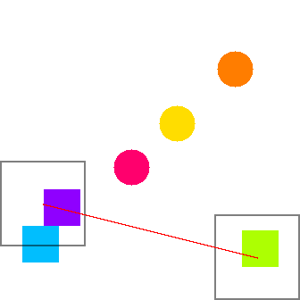}
\end{subfigure}%
\caption{Q: ``How many objects have shape of green object?''}%
\medskip%
\centering
\begin{subfigure}{0.22\textwidth}
\includegraphics[width=\linewidth, frame]{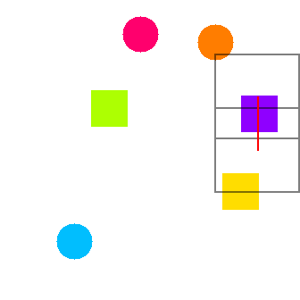}
\end{subfigure}
\hfill
\begin{subfigure}{0.22\textwidth}
\includegraphics[width=\linewidth, frame]{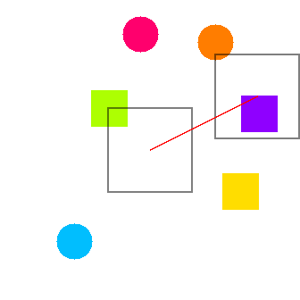}
\end{subfigure}
\hfill
\begin{subfigure}{0.22\textwidth}
\includegraphics[width=\linewidth, frame]{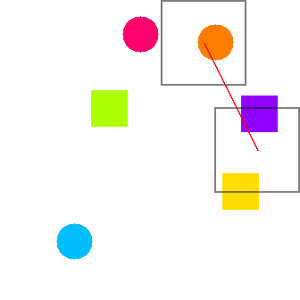}
\end{subfigure}
\hfill
\begin{subfigure}{0.22\textwidth}
\includegraphics[width=\linewidth, frame]{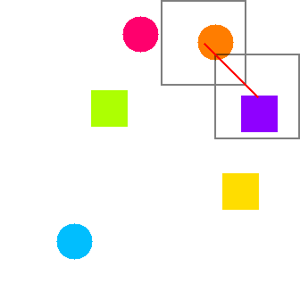}
\end{subfigure}%
\caption{Q: ``Which shape is closest to the purple square?''}%
\medskip%
\begin{subfigure}{0.22\textwidth}
\includegraphics[width=\linewidth, frame]{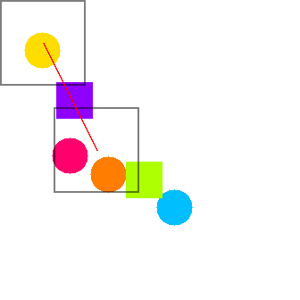}
\end{subfigure}
\hfill
\begin{subfigure}{0.22\textwidth}
\includegraphics[width=\linewidth, frame]{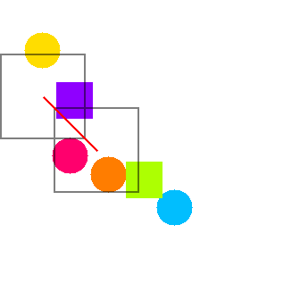}
\end{subfigure}
\hfill
\begin{subfigure}{0.22\textwidth}
\includegraphics[width=\linewidth, frame]{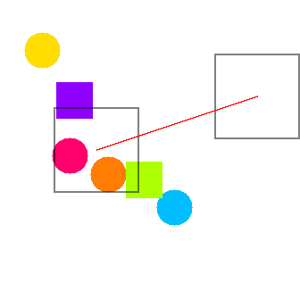}
\end{subfigure}
\hfill
\begin{subfigure}{0.22\textwidth}
\includegraphics[width=\linewidth, frame]{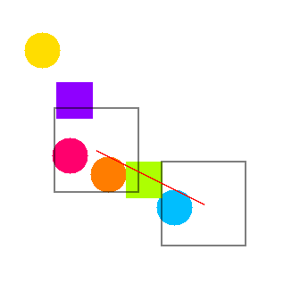}
\end{subfigure}
\caption{Q: ``Which shape is furthest from the pink circle?''}%
\medskip%
\begin{subfigure}{0.22\textwidth}
\includegraphics[width=\linewidth, frame]{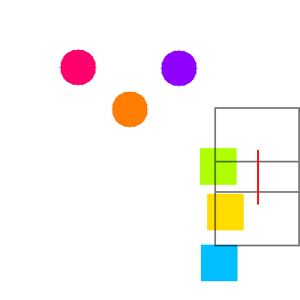}
\end{subfigure}%
\hfill
\begin{subfigure}{0.22\textwidth}
\includegraphics[width=\linewidth, frame]{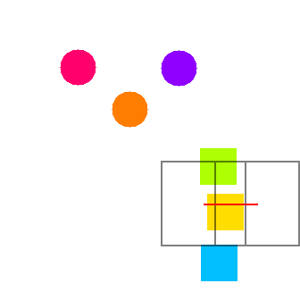}
\end{subfigure}
\hfill
\begin{subfigure}{0.22\textwidth}
\includegraphics[width=\linewidth, frame]{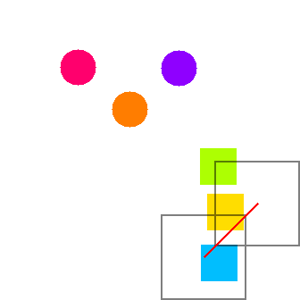}
\end{subfigure}
\hfill
\begin{subfigure}{0.22\textwidth}
\includegraphics[width=\linewidth, frame]{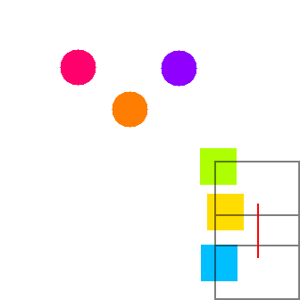}
\end{subfigure}%
\caption{Q: ``How many objects have shape of yellow object?''}%
\end{subfigure}%
\caption{Shown are the top 4 interactions identified from a Relation Network's predictions on 6 visual question-answering samples. 
The 
bounding
boxes 
proposed by Taylor-CAM may 
be interpreted as 
indicating
a relation.
We recommend testing yourself to see if you can guess (1) the object of interest and (2) the question being asked (\textit{closest}, \textit{furthest}, or \textit{same shape}), without looking at the caption. 
}
\label{soclevr}
\end{figure}
Sort-Of-CLEVR is a toy dataset for relational reasoning proposed by \cite{RN}. It is a less-computationally expensive 2D form of the CLEVR VQA dataset \cite{clevr} with a focus on relational questions. In our setup, these questions include distance 
and compare-\&-count tasks. To test Taylor-CAM's capacity to explain a neural network’s relational reasoning, we train a Relation Network \cite{RN} on Sort-Of-CLEVR and visualize its top
interactions in Figure \ref{soclevr}. Relation Networks are 
simple 
modules augmented to CNNs that enable relational reasoning between image regions. 

In Figure \ref{soclevr}, interacting regions are indicated by two bounding boxes, and the top $4$ interactions discovered by Taylor-CAM are shown per image. The input is an image of objects and a question about a particular \textit{object of interest} and its relation to another object, and the output is the answer to that question. Since these questions are relational in nature, this problem requires relational reasoning, which we hope Taylor-CAM can be suited to explain.
We invite the reader to use the discovered interactions in Figure \ref{soclevr} (as visualized by the bounding boxes) to try to deduce the objects of interest and questions for themselves before looking at the captions. For example, if the top $4$ interactions each consist of objects that are close to each other and if each interaction includes the pink square, one might guess that the question is ``Which shape is closest to the pink square?'' 


The 6 objects are ``blue'', ``purple'', ``pink'', ``yellow'', ``orange'', and ``green'' and the 3 questions are (1) ``Which shape is closest to the object of interest?'', (2) ``Which shape is furthest from the object of interest?'', and (3) ``How many objects have the same shape as the object of interest?''

While decisions are frequently relational \cite{relationalreasoning}, Grad-CAM is only designed to explain the importance of individual objects in isolation. We observed that Taylor-CAM affords much clearer explanations when decisions are relational. 

\begin{table}
\small
\centering
\caption{
Quantitative analysis on Sort-Of-CLEVR ($\%$)}
\resizebox{\linewidth}{!}{
\begin{tabular}{lccccccc}
                    \hline & Taylor-CAM & Grad-CAM\mbox{*} \cite{gradcam} & GLIDER \cite{tsang2}  \\
                    \hline Ques 1 & $\mathbf{90\%}$ & $35\%$ & $60\%$\\
            Ques 2 & $\mathbf{55\%}$ & $50\%$ & $35\%$\\
            Ques 3 & $\mathbf{60\%}$ & $40\%$ & $45\%$\\
                    \hline
                    \end{tabular}}
                    
\label{quant}
\end{table}

\textbf{Quantitative Performance} \quad 
To assess quantitatively, 20 images per question that were classified correctly by the model were randomly selected and annotated with their question's object of interest and answer-relevant objects. For example, for the question, ``What is the shape of the object closest to the green square?" the green square and the object that is closest to it are annotated. If Taylor-CAM's top-1 interaction (a pair of bounding boxes) intersects with the annotated pair, then it is counted as accurate for that image. Same with GLIDER. If Grad-CAM's top-2 saliences include the annotated pair, then it is counted as accurate for that image. Since Grad-CAM does not provide relational interpretations, we refer to this relational interpretation of Grad-CAM's saliences as Grad-CAM\mbox{*}. The bounding boxes in Figure \ref{tcgc} exemplify what a single salience looks like for Taylor-CAM and Grad-CAM respectively. Results of the quantitative analysis are reported in Table \ref{quant}.

\begin{figure}
\centering
\begin{subfigure}{.48\linewidth}
\includegraphics[width=.47\linewidth, frame]{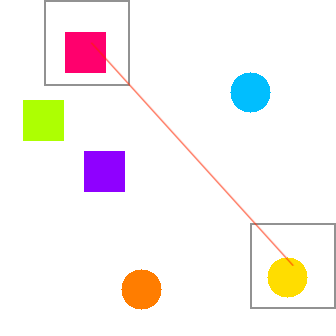}
\hfill
\includegraphics[width=.47\linewidth, frame]{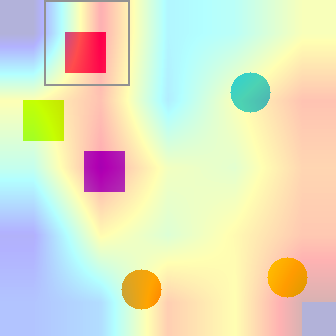}
\caption{Q: \textit{``Which shape is furthest from the pink square?''}}
\end{subfigure}
\hfill
\begin{subfigure}{.48\linewidth}
\includegraphics[width=.47\linewidth, frame]{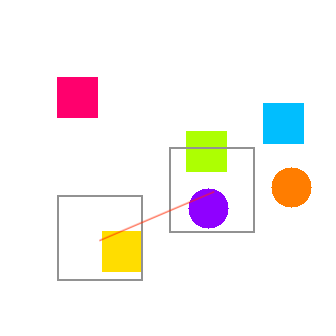}
\hfill
\includegraphics[width=.47\linewidth, frame]{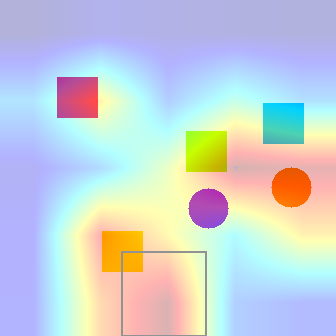}
\caption{Q: \textit{``Which shape is closest to the yellow square?''}}
\end{subfigure}
\caption{The bounding boxes show what a top-1 salience looks like for Taylor-CAM (on the left) and Grad-CAM (on the right) respectively. Taylor-CAM offers interpretable relational explanations from a single top-1 salience, whereas Grad-CAM depends on all saliences to produce a non-relational heatmap. 
}
\label{tcgc}
\end{figure}


\textbf{Qualitative Performance} \quad To measure Taylor-CAM's qualitative explainability, we selected a random batch of $15$ samples and their ordered interaction saliences, and conducted a small user study ($n=10$), asking each individual to guess (1) the object of interest and (2) the question being asked, from just looking at the top-4 ranked interaction visuals. 
Taylor-CAM achieves
strong explainability with better guess-accuracy than Grad-CAM and the recent GLIDER \cite{tsang2}. With Taylor-CAM, participants were able to reverse engineer questions in relational VQA from just looking at the visualized interactions.
We report a wide range of explainability across different colors and questions in Tables \ref{objects} and \ref{questions}. 
Due to random sampling, none of the $15$ sampled images for Grad-CAM included a purple object of interest, so it is marked ``N/A'' in Table \ref{objects}. 
\begin{table}
\small
\centering
\caption{User study \textbf{object of interest} accuracy ($\%$)}
\resizebox{\linewidth}{!}{
\begin{tabular}{lccc}
                    \hline & Grad-CAM \cite{gradcam} & GLIDER \cite{tsang2} & Taylor-CAM \\
                    \hline Green & $13.3\%$ & $33.3\%$ & \bm{$40\%$}\\
            Pink & $30\%$ & $10\%$ & \bm{$46.7\%$}\\
            Blue & $10\%$ & $22.2\%$ & \bm{$40\%$}\\
            Purple& N/A & \bm{$15\%$} & $10\%$ \\
            Orange & $3.3\%$ & $10\%$ & \bm{$15\%$}\\
            Yellow & $25\%$ & $16.7\%$ & \bm{$33.3\%$}\\
                    \hline
                    \end{tabular}}
\label{objects}
\end{table}
\begin{table}
\small
\centering
\caption{User study \textbf{question} accuracy ($\%$)}
\resizebox{\linewidth}{!}{
\begin{tabular}{lccc}
                    \hline & Grad-CAM \cite{gradcam} & GLIDER \cite{tsang2} & Taylor-CAM \\
                    \hline Ques 1 & $44\%$ & $38.9\%$ & \bm{$76\%$}\\
            Ques 2 & $14\%$ & $38.9\%$ & \bm{$55\%$}\\
            Ques 3 & $30\%$ & $23.8\%$ & \bm{$48.3\%$}\\
                    \hline
                    \end{tabular}
}
\label{questions}
\end{table}

While some Grad-CAM colors strongly outperform random guessing (pink and yellow), on average, people struggled guessing the object of interest with Grad-CAM. 
This is because Grad-CAM only explains which individual objects contribute to the output, which in 
relational VQA, is all of them with an equal importance assigned to the object of interest and any objects that are included in the question-answer, such as the furthest or nearest object. This results in uninterpretable and sometimes misleading visualizations, making it very hard 
to guess an object of interest 
from the visual only. Without knowing the object of interest, it is consequently much harder to guess the question asked.

Grad-CAM, GLIDER, and Taylor-CAM all did relatively well on question 1.
Closeness is easier to interpret with all three explanatory tools, since it is usually more visually apparent. However, we found question 2 (furthest distance) to be harder to interpret for Grad-CAM, perhaps because it is unclear what the object of interest is, with multiple ``far away'' objects of different relative proximity being ranked highly. For example, two objects that are far away from the object of interest might be close to each other, creating the false impression that the question is asking about closeness. Thus, without confidence regarding the object of interest and the interacting parts, we found ranked importances alone to be unintuitive and even misleading.

\subsection{Biomedical Application} 

We also 
applied T-NID to determine 
interactions in
the PPMI study dataset 
(\url{www.ppmi-info.org}).
Our analysis suggests that various measures previously thought to be unrelated should be considered together when predicting faster cognitive progression in 
Parkinson's disease.
Please see Appendix for 
details
in this domain.

\section{Architecture Configurations}

\textbf{T-NID} \quad For T-NID, we trained a GELU-activated multi-layer perceptron with hidden layer sizes 140, 100, 60, and 20 for 200 epochs with a learning rate of 0.003 using early stopping \cite{earlystop} with a patience of 10. Results were averaged across 10 trials. Input data was normalized by standard deviation. T-NID hyperparameters were set as $\ell = 5, o = 2, k = 10$.

\textbf{Taylor-CAM} \quad For our COCO \cite{coco} task, we used Pytorch's ResNet-50 \cite{resnet} pretrained on ImageNet \cite{imagenet}, except we replaced the global average pooling layer with an additional convolutional layer composed of 1024 out-channels, size 2 kernel, 2 stride, and 2 padding, followed by 3 hidden linear layers of size 512, 256, 64, because global average pooling yields no higher-order derivatives. For our Relation Network, we used an open source reference implementation, which can be found here: \url{https://github.com/kimhc6028/relational-networks}, since \cite{RN} did not release code to the public. We trained for 50 epochs.
\section{Conclusion}
With T-NID and Taylor-CAM, we have shown that input cross derivatives, combined with a few simple heuristics and intuitions, are a powerful tool for explaining interactions in deep learning. T-NID, using GELU activations, representative samples, and interaction subsampling, successfully ranks statistical interactions, outperforming NID. Meanwhile, Taylor-CAM generalizes Grad-CAM to the higher order and effectively explains interactions in object detection and relational reasoning, affording a user cohort the insight to guess questions in VQA from only seeing the top discovered visual interactions. 
Future work may explore localizing multi-modal interactions such as in audio-visual tasks, an agent's interactions in RL and robotics, and 
interactions between 
word 
embeddings
in NLP.
By making our code publicly available, we hope that these simple explanatory tools can be used and built upon to better explain the complex interoperating factors underlying neural network reasoning and the world.

\section{Acknowledgments}


Research reported in this publication was supported by the National Science Foundation (NSF) under Grants 1741472, 1813709, 1909912, and 1934962 and by the National Institutes of Health under Award Number P50NS108676. The content is solely the responsibility of the authors and does not necessarily represent the official views of the funding agents.


{\small
\bibliographystyle{ieee_fullname}
\bibliography{egbib}
}

\appendix
\section*{Appendix}

\addcontentsline{toc}{section}{Appendix}
\startcontents[sections]
\printcontents[sections]{l}{1}{\setcounter{tocdepth}{2}}



\section{Note On Terminology}

In colloquial terms, two things are said to interact when they depend on each other in some way. Similar to \cite{friedman}, this can be formalized as follows: 

\begin{definition}{\textbf{Entity Interaction}}
\label{coloquial}
Given an entity $e_1$ with attributes $(a_1, ..., a_n)$, an interaction exists with another entity $e_2$ with attributes $(b_1, ..., b_n)$ if some $a_i$ depends on some $b_j$ or some $b_j$ depends on some $a_i$. 
\end{definition}

Now we will define mathematical relation.

\begin{definition}{\textbf{Relation}}
\label{mathrel}
Given sets $A$ and $B$, the binary relation from $A$ to $B$ is a subset of the Cartesian product $A \times B$. 
\end{definition}

We would like to unify our colloquial understanding of interaction in Definition \ref{coloquial}, our mathematical definition of relation in Definition \ref{mathrel}, and our definition for statistical interaction effects in Definition 1 of the main paper.

To connect this to Definition 1, we will reframe features as entities with the following theorem:

\begin{theorem}
Given a function $F(\bm{x})$ and feature $x_i$, let entity $e_i$ consist of attributes $(x_i, F/\partial x_i)$. 
An interaction exists between $e_1$ and $e_2$ if there is a nonzero interaction effect between $x_1$ and $x_2$.
\end{theorem}

\begin{proof}
If there is a nonzero interaction effect between $x_1$ and $x_2$, then $\partial^2 F(x) / (\partial x_1 \partial x_2) \neq 0$ for some input $\bm{x}$. Then $F/\partial x_1$ depends on $x_2$ and consequently, there exists an interaction between entities $e_1$ and $e_2$.
\end{proof}

We have shown that our statistical interaction implies an interaction according to our colloquial understanding. An interaction exists between $e_1$ and $e_2$ if (but not only if, since the change need not be local) $F(\bm{x})/(\partial x_1 \partial x_2) \neq 0$, meaning $F/\partial x_1$ depends on $x_2$. This is considered a binary relation between the two attributes, as all functions are relations, though not all relations are functions. Formally: given a function $F(\bm{x})$, a feature $x_i$, and entity $e_i$ consisting of attributes $(x_i, F/dx_i)$, if there is a nonzero interaction effect between $x_1$ and $x_2$, then a relation exists between the attributes of the two entities.

We have shown that, under this framing, an interaction effect is a relation, and if the interaction effect is nonzero, there must be a dependency/interaction between those entities. Since feature vectors in CNNs could be treated as entities \cite{RN, rdrl, rrnn}, and if one interprets their gradients on the output to be implicit attributes, computing interaction effects between CNN feature vectors is equivalent to identifying the colloquial interactions and relations described in this formulation.

This is trivially generalized to interactions/relations of higher orders. 

To summarize, a mathematical relation is implied by a colloquial interaction is implied by a statistical interaction, and this hierarchy can be formalized by regarding a feature $x_i$ as an entity whose attributes include its gradients with respect to the function of interest. Thus, we offer a simple, formal connection between our statistical interaction effects definition and mathematical relations, as well as an integration of both into the colloquial understanding of “interaction” as merely a dependency between two “things.”

\section{Representative Samples \& Aggregations}


\begin{table}[h]
  \centering
\begin{tabular}{lc}
\hline 
Aggregation Of Representative Samples & AUC Score\\
\hline Mean Of Mean-Min-Mode-Rand & 0.61825 \\
Mean Of Med-Min-Mode-Rand & 0.61825 \\
Mean Of Mean-Med-Min-Mode-Rand & 0.61775 \\
Mean Of Mean-Min-Max-Mode-Rand & 0.6155 \\
Mean Of Med-Min-Max-Mode-Rand & 0.6155 \\
Med Of Mean-Min-Mode-Rand & 0.61525 \\
Med Of Med-Min-Mode--Rand & 0.61525 \\
Mean Of Mean-Med-Min-Max-Mode-Rand & 0.61525 \\
Mean Of Mean-Min-Rand & 0.614 \\
Mean Of Med-Min-Rand & 0.614 \\
\hline
\end{tabular}
  \caption{Top average (across all orders) AUC scores for different aggregations of representative samples}
  \label{aggs}
\end{table}

Table \ref{aggs} displays the top 10 aggregations and representative samples discovered via our power sweep. 


\section{T-NID Algorithm}

Our complete T-NID is described in Algorithm \ref{tnid}. Note that each derivation of interaction effect using Definition 1 of the main paper for an interaction $I = \hat{I} \cup j$ of size $\hat{\ell}$ where $|\hat{I}| = \hat{\ell} - 1$ for sample $x$ can be derived as a single-order partial derivative $\partial \mathit{\bm{IE}}_{\hat{I}}/\partial x_j$ and does not need to be recomputed from the ground up.

\begin{algorithm}[h]
\textbf{Inputs} $\ell$-times differentiable trained neural network $F$, dataset $\bm{X}$ with $i$th sample features $\bm{X}_{i1}, ..., \bm{X}_{in}$, order $\ell$, orders without subsampling $o$, subsampling size $k$.  

\textbf{Outputs} Interaction effects $\bm{\mathit{IE}}_I$ for top estimated interactions $I \subseteq \{1, ..., n\}$, where $|I| \leq \ell$.
 \\ \\
Get representative samples:

\textbf{For} {$j$th $\mathit{aggregation}  \in$ mean, minimum, mode, random}
    
\hspace*{5mm} $c = \argmin_{i}{\left\lVert \bm{X}_i - \mathit{aggregation}(\bm{X}, axis=0)\right\rVert}$
    
\hspace*{5mm} $\bm{r}_{j} = \bm{X}_c$ 

For each representative sample:

\textbf{For} {$\bm{r}_{j} \in \bm{r}$} 

\hspace*{5mm}Compute all non-redundant partial derivatives up to order $o$:
        
\hspace*{5mm}\textbf{For} $I \subseteq \{1, ..., n\}$, where $|I| \leq o$ 

\hspace*{10mm} $I = \mathit{sort}(I)$

\hspace*{10mm}\textbf{If} $\bm{\mathit{IE}}^{(j)}_I$ uninitiated

\hspace*{15mm}\textbf{Initiate} $\bm{\mathit{IE}}^{(j)}_I$ according to Definition 1 of the main paper

\hspace*{5mm}Compute remaining partial derivatives up to order $\ell$ by subsampling top $k$ from previous orders:

\hspace*{5mm}\textbf{For} $\hat{\ell} \in o + 1, ..., \ell$

\hspace*{10mm}\textbf{For} $\hat{I} \in$ top $k$ argmax of $\bm{\mathit{IE}}^{(j)}_I$, where $|I| = \hat{\ell} - 1$ 

\hspace*{15mm}\textbf{For} $I \subseteq \{1, ..., n\}$, where $|I| = \ell$ and $\hat{I} \subset I$

\hspace*{20mm}\textbf{If} $\bm{\mathit{IE}}^{(j)}_I$ uninitiated

\hspace*{25mm}\textbf{Initiate} $\bm{\mathit{IE}}^{(j)}_I$ according to Definition 1 of the main paper

Take the mean interaction effects across representative samples:

\textbf{For} $I \subseteq \{1, ..., n\}$ if $\bm{\mathit{IE}}^{(j)}_I$ initiated for some $j$

\hspace*{5mm} $\bm{\mathit{IE}}_I = \mathit{mean}(\bm{\mathit{IE}}^{(j)}_I)$ for all $j$ where $\bm{\mathit{IE}}^{(j)}_I$ initiated

\textbf{Return} $\bm{\mathit{IE}}$

\caption{T-NID algorithm in pseudocode}
\label{tnid}
\end{algorithm}

\section{Test Suite Of Synthetic Functions}

The test-suite of synthetic functions used to evaluate T-NID may be found in Table \ref{syntheticfuncs}, courtesy of \cite{NID}.

\begin{table*}[h]
\setlength{\tabcolsep}{3pt}
  \centering
  \small
\begin{tabularx}{\textwidth}{cc}
\hline$F_{1}(\mathrm{x})$ & $\pi^{x_{1} x_{2}} \sqrt{2 x_{3}}-\sin ^{-1}\left(x_{4}\right)+\log \left(x_{3}+x_{5}\right)-\frac{x_{9}}{x_{10}} \sqrt{\frac{x_{7}}{x_{8}}}-x_{2} x_{7}$ \\
\hline $F_{2}(\mathrm{x})$ & $\pi^{x_{1} x_{2}} \sqrt{2\left|x_{3}\right|}-\sin ^{-1}\left(0.5 x_{4}\right)+\log \left(\left|x_{3}+x_{5}\right|+1\right)+\frac{x_{9}}{1+\left|x_{10}\right|} \sqrt{\frac{\left|x_{7}\right|}{1+\left|x_{8}\right|}-x_{2} x_{7}}$ \\
\hline $F_{3}(\mathrm{x})$ & $\exp \left|x_{1}-x_{2}\right|+\left|x_{2} x_{3}\right|-x_{3}^{2\left|x_{4}\right|}+\log \left(x_{4}^{2}+x_{5}^{2}+x_{7}^{2}+x_{8}^{2}\right)+x_{9}+\frac{1}{1+x_{10}^{2}}$ \\
\hline $F_{4}(\mathrm{x})$ & $\exp \left|x_{1}-x_{2}\right|+\left|x_{2} x_{3}\right|-x_{3}^{2\left|x_{4}\right|}+\left(x_{1} x_{4}\right)^{2}+\log \left(x_{4}^{2}+x_{5}^{2}+x_{7}^{2}+x_{8}^{2}\right)+x_{9}+\frac{1}{1+x_{10}^{2}}$ \\
\hline $F_{5}(\mathrm{x})$ & $\frac{1}{1+x_{1}^{2}+x_{2}^{2}+x_{3}^{2}}+\sqrt{ \left|x_{4}+x_{5}\right|}+\left|x_{6}+x_{7}\right|+x_{8} x_{9} x_{10}$ \\
\hline $F_{6}(\mathrm{x})$ & $\exp \left(\left|x_{1} x_{2}\right|+1\right)-\exp \left(\left|x_{3}+x_{4}\right|+1\right)+\cos \left(x_{5}+x_{6}-x_{8}\right)+\sqrt{x_{8}^{2}+x_{9}^{2}+x_{10}^{2}}$ \\
\hline $F_{7}(\mathrm{x})$ & $\left(\arctan \left(x_{1}\right)+\arctan \left(x_{2}\right)\right)^{2}+\max \left(x_{3} x_{4}+x_{6}, 0\right)-\frac{1}{1+\left(x_{4} x_{5} x_{6} x_{7} x_{8}\right)^{2}}+\left(\frac{\left|x_{7}\right|}{1+\left|x_{9}\right|}\right)^{5}+\sum_{i=1}^{10} x_{i}$ \\
\hline $F_{8}(\mathrm{x})$ & $x_{1} x_{2}+2^{x_{3}+x_{5}+x_{6}}+2^{x_{3}+x_{4}+x_{5}+x_{7}}+\sin \left(x_{7} \sin \left(x_{8}+x_{9}\right)\right)+\arccos \left(0.9 x_{10}\right)$ \\
\hline $F_{9}(\mathrm{x})$ & $\tanh \left(x_{1} x_{2}+x_{3} x_{4}\right) \sqrt{\left|x_{5}\right|}+\exp \left(x_{5}+x_{6}\right)+\log \left(\left(x_{6} x_{7} x_{8}\right)^{2}+1\right)+x_{9} x_{10}+\frac{1}{1+\left|x_{10}\right|}$ \\
\hline $F_{10}(\mathrm{x})$ & $\sinh \left(x_{1}+x_{2}\right)+\arccos \left(\tanh \left(x_{3}+x_{5}+x_{7}\right)\right)+\cos \left(x_{4}+x_{5}\right)+\sec \left(x_{7} x_{9}\right)$ \\
\hline
\end{tabularx}
  \caption{Synthetic test-suite functions}
  \label{syntheticfuncs}
\end{table*}

\section{Additional Architectures For $N$-Way Interactions}

Table \ref{mlpm} shows results for T-NID + MLP-M (T-NID using a neural network equipped with a main effects network as well as trained with sparsity regularization) and NID + MLP-M, the architecture used in \cite{NID}. 

\begin{table*}[h]
\vspace{1px}
\begin{minipage}{.1023\linewidth}
\centering
\resizebox{\linewidth}{!}{
\begin{tabular}{l}
                    \hline \\
                    \hline $F_{1}(\mathbf{x})$  \\
            $F_{2}(\mathbf{x})$ \\
            $F_{3}(\mathbf{x})$ \\
            $F_{4}(\mathbf{x})$ \\
            $F_{5}(\mathbf{x})$ \\
            $F_{6}(\mathbf{x})$ \\
            $F_{7}(\mathbf{x})$ \\
            $F_{8}(\mathbf{x})$ \\
            $F_{9}(\mathbf{x})$ \\
            $F_{10}(\mathbf{x})$ \\
            \hline average \\
                    \hline
                    \end{tabular}
}
\end{minipage}
\hspace{-10px}
\begin{minipage}{.31\linewidth}
\centering
\resizebox{\linewidth}{!}{
\begin{tabular}{cc}
                    \hline T-NID 3-Way & NID 3-Way \\
                    \hline $0.831 \pm 0.064$ & $0.122 \pm 0.028$ \\
            $0.629 \pm 0.165$ & $0.07 \pm 0.011$ \\
            $0.991 \pm 0.013$ & $0.095 \pm 0.008$ \\
            $0.993 \pm 0.007$ & $0.09 \pm 0.028$ \\
            $0.493 \pm 0.009$ & $0.035 \pm 0.005$ \\
            $0.103 \pm 0.025$ & $0.034 \pm 0.005$ \\
            $0.417 \pm 0.264$ & $0.156 \pm 0.031$ \\
            $1.0 \pm 0.0$ & $0.141 \pm 0.008$ \\
            $0.838 \pm 0.146$ & $0.113 \pm 0.01$ \\
            $1.0 \pm 0.0$ & $0.03 \pm 0.002$ \\
            \hline $0.73 \pm 0.069$ & $0.089 \pm 0.014$ \\
                    \hline
                    \end{tabular}
}
\end{minipage}
\hspace{-10px}
\begin{minipage}{.31\linewidth}
\centering
\resizebox{\linewidth}{!}{
\begin{tabular}{cc}
                    \hline T-NID 4-Way & NID 4-Way \\
                    \hline $0.777 \pm 0.389$ & $0.555 \pm 0.456$ \\
            $0.032 \pm 0.065$ & $0.185 \pm 0.37$ \\
            $0.997 \pm 0.006$ & $1.0 \pm 0.0$ \\
            $0.96 \pm 0.08$ & $0.996 \pm 0.009$ \\
            N/A & N/A \\
            N/A & N/A \\
            $0.363 \pm 0.322$ & $0.711 \pm 0.046$ \\
            $1.0 \pm 0.0$ & $0.994 \pm 0.008$ \\
            $0.859 \pm 0.068$ & $0.618 \pm 0.084$ \\
            N/A & N/A \\
            \hline $0.713 \pm 0.133$ & $0.723 \pm 0.139$ \\
                    \hline
                    \end{tabular}
}
\end{minipage}
\hspace{-10px}
\begin{minipage}{.31\linewidth}
\centering
\resizebox{\linewidth}{!}{
\begin{tabular}{cc}
                    \hline T-NID 5-Way & NID 5-Way \\
                    \hline N/A & N/A \\
            N/A & N/A \\
            N/A & N/A \\
            N/A & N/A \\
            N/A & N/A \\
            N/A & N/A \\
            $0.303 \pm 0.367$ & $0.536 \pm 0.453$ \\
            N/A & N/A \\
            $0.988 \pm 0.024$ & $0.549 \pm 0.452$ \\
            N/A & N/A \\
            \hline $0.646 \pm 0.200$ & $0.543 \pm 0.453$ \\
                    \hline
                    \end{tabular}
}
\end{minipage}
\caption{N-Way AUC scores for T-NID + MLP-M and NID + MLP-M, both using a main effects network and sparsity regularization, as described in \cite{NID}}
\label{mlpm}
\end{table*}

\begin{figure*}
\centering
\includegraphics[width=\linewidth]{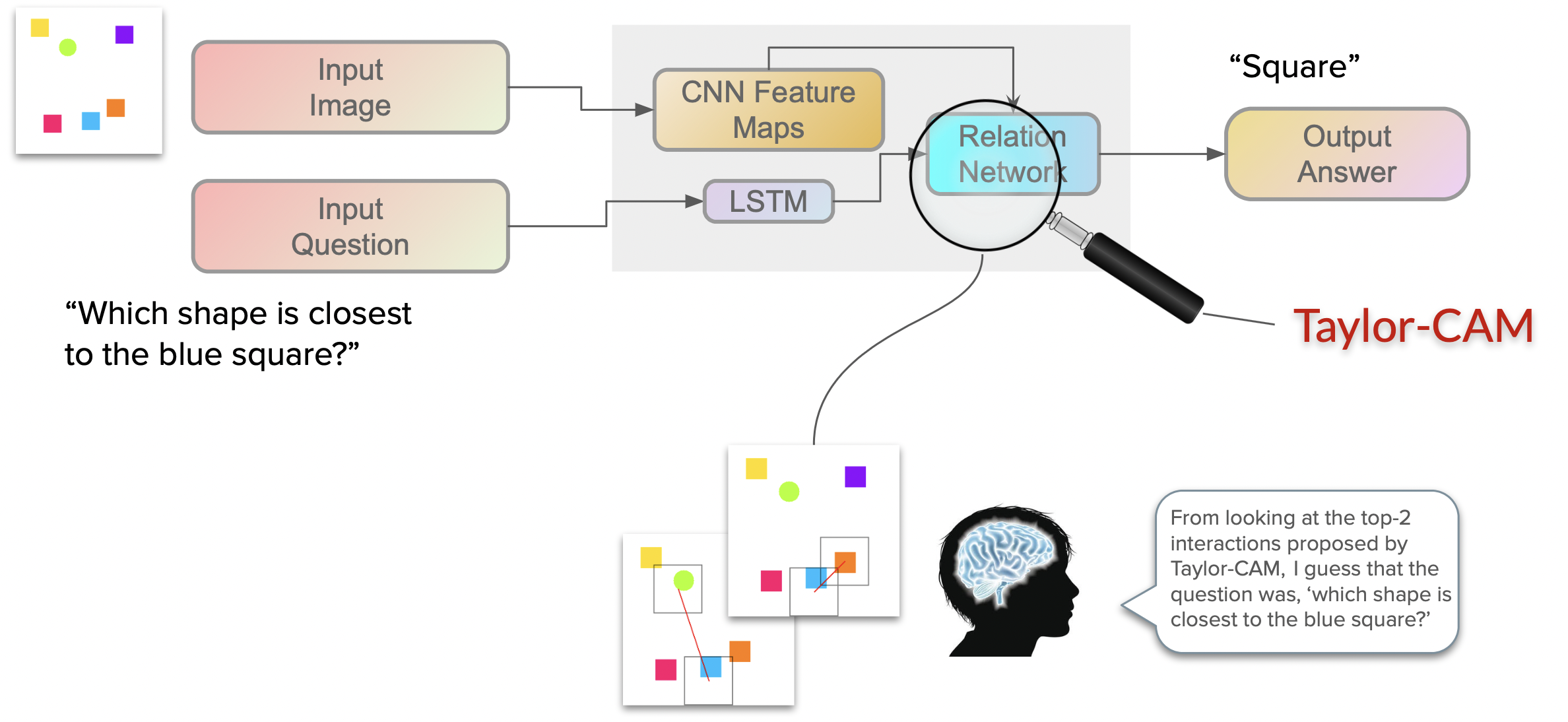}
\caption{Full Taylor-CAM pipeline on the problem of relational visual question answering. A CNN + RN take in an image and question, and output an answer. Taylor-CAM is able to visualize the model's reasoning from just its gradients such that a human can interpret what and how the model reasoned.}
\label{pipeline}
\end{figure*}

\section{COCO Multi-Object Detection}
\begin{figure*}
\centering
\begin{subfigure}{.98\textwidth}
\centering
\begin{subfigure}{0.24\textwidth}
\includegraphics[width=\linewidth, frame]{COCO_In_Depth/0/0.png}
\end{subfigure}
\hfill
\begin{subfigure}{0.24\textwidth}
\includegraphics[width=\linewidth, frame]{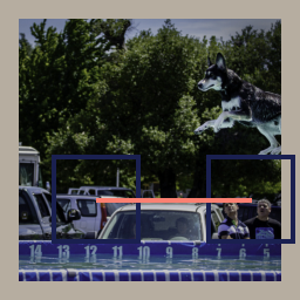}
\end{subfigure}
\hfill
\begin{subfigure}{0.24\textwidth}
\includegraphics[width=\linewidth, frame]{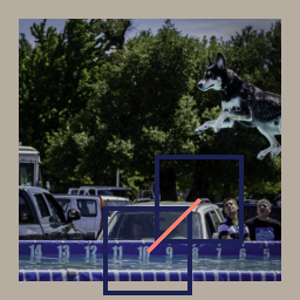}
\end{subfigure}
\hfill
\begin{subfigure}{0.24\textwidth}
\includegraphics[width=\linewidth, frame]{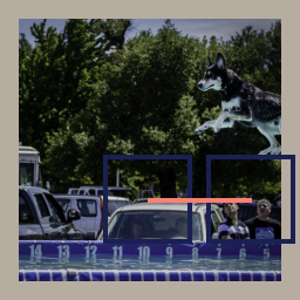}
\end{subfigure}
\vfill
\medskip
\begin{subfigure}{0.24\textwidth}
\includegraphics[width=\linewidth, frame]{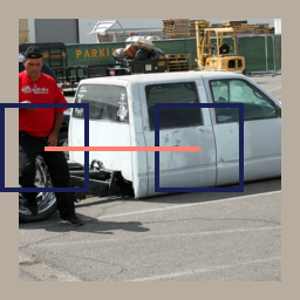}
\end{subfigure}
\hfill
\begin{subfigure}{0.24\textwidth}
\includegraphics[width=\linewidth, frame]{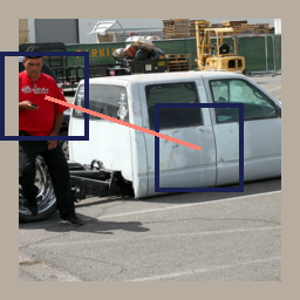}
\end{subfigure}
\hfill
\begin{subfigure}{0.24\textwidth}
\includegraphics[width=\linewidth, frame]{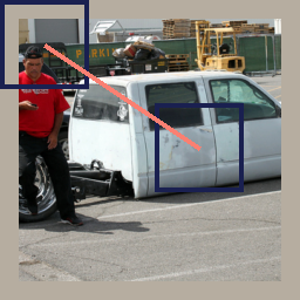}
\end{subfigure}
\hfill
\begin{subfigure}{0.24\textwidth}
\includegraphics[width=\linewidth, frame]{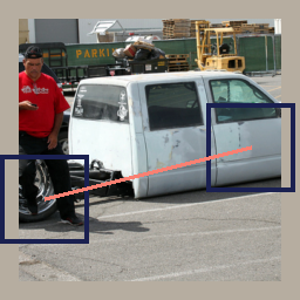}
\end{subfigure}
\vfill
\medskip
\begin{subfigure}{0.24\textwidth}
\includegraphics[width=\linewidth, frame]{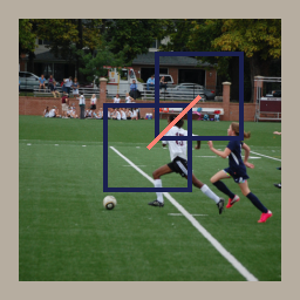}
\end{subfigure}
\hfill
\begin{subfigure}{0.24\textwidth}
\includegraphics[width=\linewidth, frame]{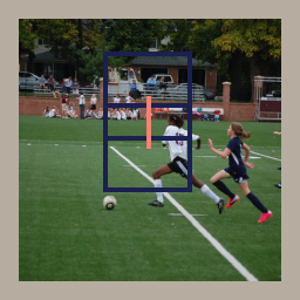}
\end{subfigure}
\hfill
\begin{subfigure}{0.24\textwidth}
\includegraphics[width=\linewidth, frame]{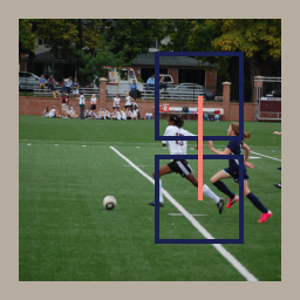}
\end{subfigure}
\hfill
\begin{subfigure}{0.24\textwidth}
\includegraphics[width=\linewidth, frame]{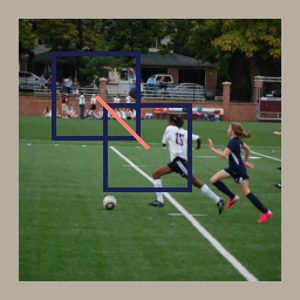}
\end{subfigure}
\vfill
\medskip
\begin{subfigure}{0.24\textwidth}
\includegraphics[width=\linewidth, frame]{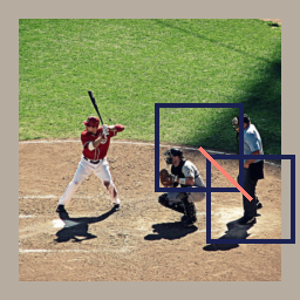}
\end{subfigure}
\hfill
\begin{subfigure}{0.24\textwidth}
\includegraphics[width=\linewidth, frame]{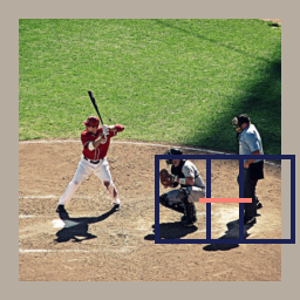}
\end{subfigure}
\hfill
\begin{subfigure}{0.24\textwidth}
\includegraphics[width=\linewidth, frame]{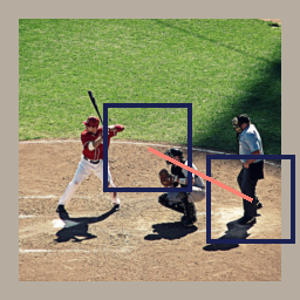}
\end{subfigure}
\hfill
\begin{subfigure}{0.24\textwidth}
\includegraphics[width=\linewidth, frame]{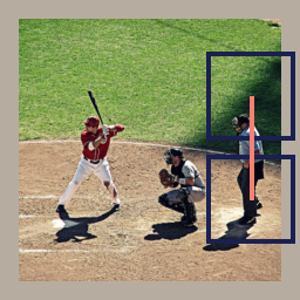}
\end{subfigure}
\vfill
\medskip
\begin{subfigure}{0.24\textwidth}
\includegraphics[width=\linewidth, frame]{COCO_In_Depth/4/0.png}
\end{subfigure}
\hfill
\begin{subfigure}{0.24\textwidth}
\includegraphics[width=\linewidth, frame]{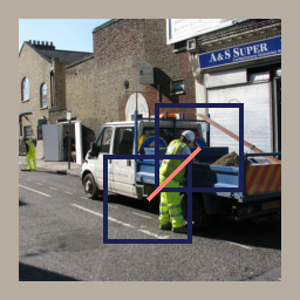}
\end{subfigure}
\hfill
\begin{subfigure}{0.24\textwidth}
\includegraphics[width=\linewidth, frame]{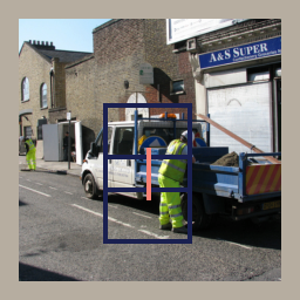}
\end{subfigure}
\hfill
\begin{subfigure}{0.24\textwidth}
\includegraphics[width=\linewidth, frame]{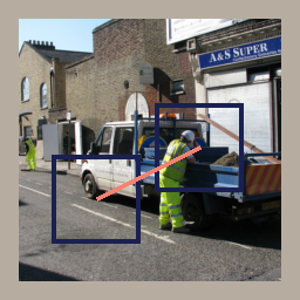}
\end{subfigure}
\end{subfigure}
\end{figure*}
    \begin{figure*}\ContinuedFloat
    \centering
    \begin{subfigure}{.92\textwidth}
\centering
\begin{subfigure}{0.24\textwidth}
\includegraphics[width=\linewidth, frame]{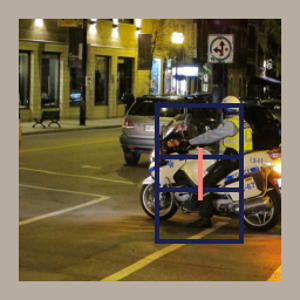}
\end{subfigure}
\hfill
\begin{subfigure}{0.24\textwidth}
\includegraphics[width=\linewidth, frame]{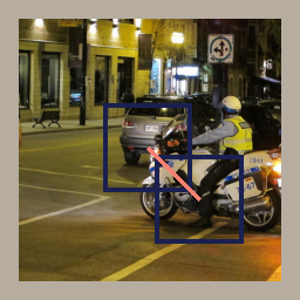}
\end{subfigure}
\hfill
\begin{subfigure}{0.24\textwidth}
\includegraphics[width=\linewidth, frame]{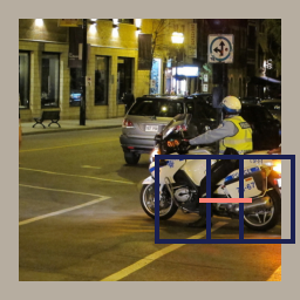}
\end{subfigure}
\hfill
\begin{subfigure}{0.24\textwidth}
\includegraphics[width=\linewidth, frame]{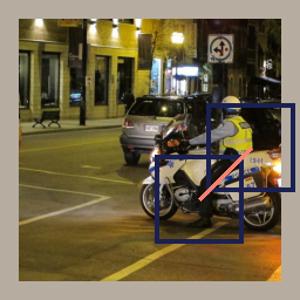}
\end{subfigure}
\vfill
\medskip
\begin{subfigure}{0.24\textwidth}
\includegraphics[width=\linewidth, frame]{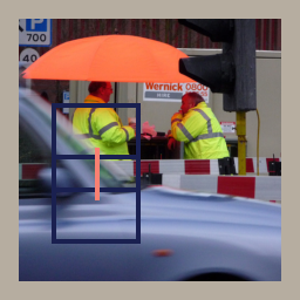}
\end{subfigure}
\hfill
\begin{subfigure}{0.24\textwidth}
\includegraphics[width=\linewidth, frame]{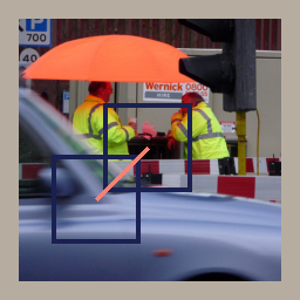}
\end{subfigure}
\hfill
\begin{subfigure}{0.24\textwidth}
\includegraphics[width=\linewidth, frame]{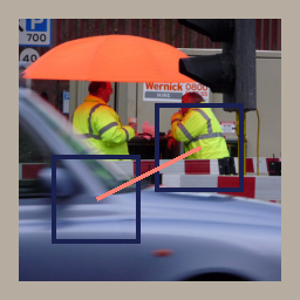}
\end{subfigure}
\hfill
\begin{subfigure}{0.24\textwidth}
\includegraphics[width=\linewidth, frame]{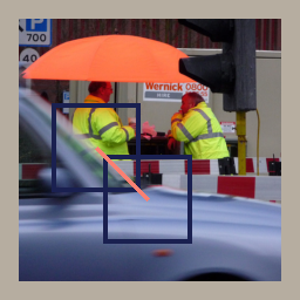}
\end{subfigure}
\vfill
\medskip
\begin{subfigure}{0.24\textwidth}
\includegraphics[width=\linewidth, frame]{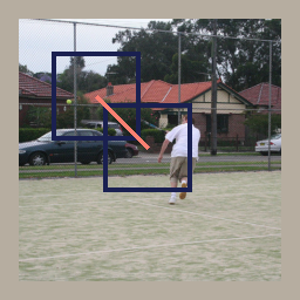}
\end{subfigure}
\hfill
\begin{subfigure}{0.24\textwidth}
\includegraphics[width=\linewidth, frame]{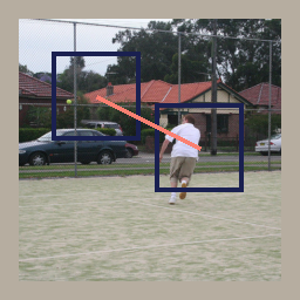}
\end{subfigure}
\hfill
\begin{subfigure}{0.24\textwidth}
\includegraphics[width=\linewidth, frame]{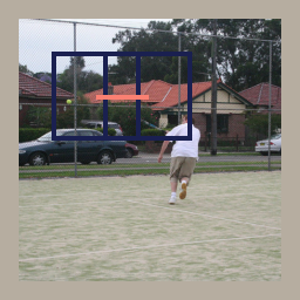}
\end{subfigure}
\hfill
\begin{subfigure}{0.24\textwidth}
\includegraphics[width=\linewidth, frame]{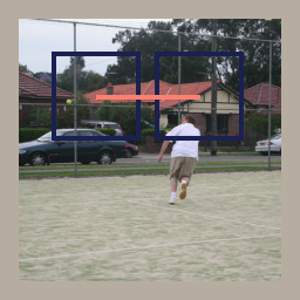}
\end{subfigure}
\vfill
\medskip
\begin{subfigure}{0.24\textwidth}
\includegraphics[width=\linewidth, frame]{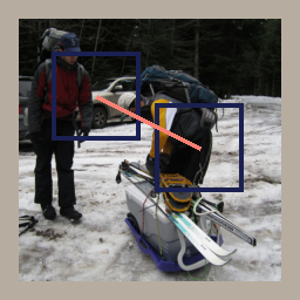}
\end{subfigure}
\hfill
\begin{subfigure}{0.24\textwidth}
\includegraphics[width=\linewidth, frame]{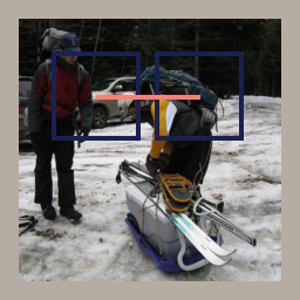}
\end{subfigure}
\hfill
\begin{subfigure}{0.24\textwidth}
\includegraphics[width=\linewidth, frame]{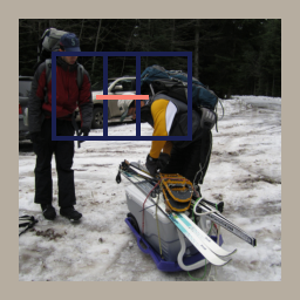}
\end{subfigure}
\hfill
\begin{subfigure}{0.24\textwidth}
\includegraphics[width=\linewidth, frame]{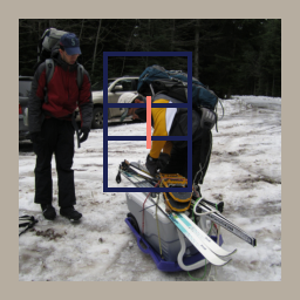}
\end{subfigure}
\vfill
\medskip
\begin{subfigure}{0.24\textwidth}
\includegraphics[width=\linewidth, frame]{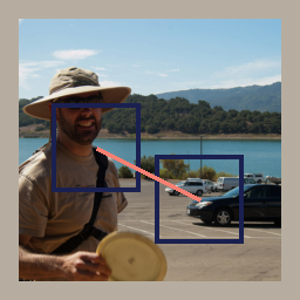}
\end{subfigure}
\hfill
\begin{subfigure}{0.24\textwidth}
\includegraphics[width=\linewidth, frame]{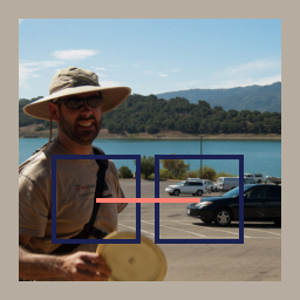}
\end{subfigure}
\hfill
\begin{subfigure}{0.24\textwidth}
\includegraphics[width=\linewidth, frame]{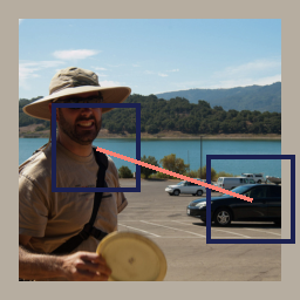}
\end{subfigure}
\hfill
\begin{subfigure}{0.24\textwidth}
\includegraphics[width=\linewidth, frame]{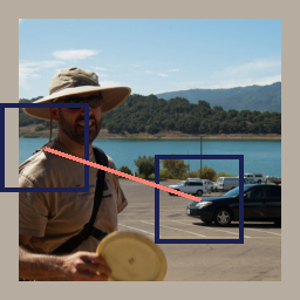}
\end{subfigure}
\end{subfigure}
\caption{Taylor-CAM interacts ``car'' and ``person'' in the custom COCO multi-object detection task. Taylor-CAM's top 4 discovered interactions per image are shown. As discussed in the main paper, in cases where the interaction is not present (as in 4th example), Taylor-CAM interacts immediately adjacent regions around whichever of the two objects is present.}
\label{objectdetection}
\end{figure*}

The task is to identify whether a pair of objects are each present in tandem. If only one is present, then the class label is negative. We tested this on the objects ``car'' and ``person'' in the COCO annotated-image dataset. We configured the frequency of the labels such that an even amount of positive and negative samples were in the training set. We found the COCO task to be somewhat inconclusive, because of model overfitting and rather low test accuracy, but still observed reasonable explanations, as seen in Figure \ref{objectdetection}. Taylor-CAM often prioritizes the car-person interaction correctly.

\section{Grad-CAM On Relational Reasoning}
\begin{figure*}
\centering
\begin{subfigure}{.92\textwidth}
\centering
\begin{subfigure}{0.45\textwidth}
\includegraphics[width=.48\linewidth, frame]{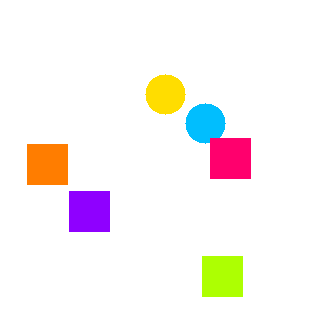}
\hfill
\includegraphics[width=.48\linewidth, frame]{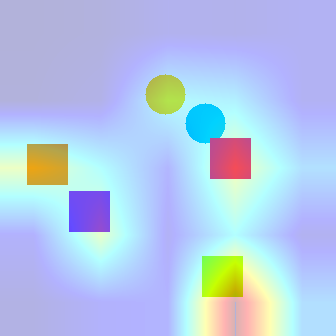}
\caption{Q: ``How many objects have shape of purple object?''}
\end{subfigure}
\hfill
\begin{subfigure}{0.45\textwidth}
\includegraphics[width=.48\linewidth, frame]{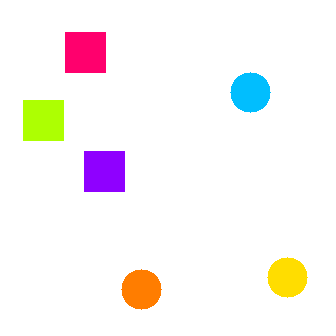}
\hfill
\includegraphics[width=.48\linewidth, frame]{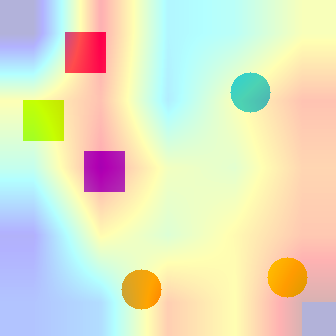}
\caption{Q: ```Which shape is furthest from the pink square?''}
\end{subfigure}
\vfill
\medskip
\begin{subfigure}{0.45\textwidth}
\includegraphics[width=.48\linewidth, frame]{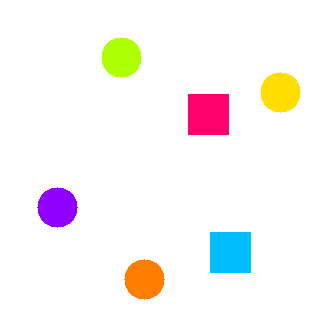}
\hfill
\includegraphics[width=.48\linewidth, frame]{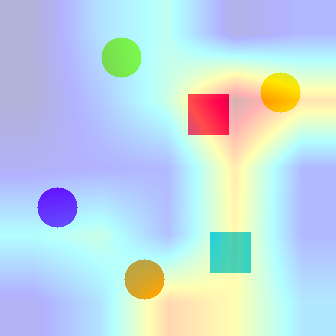}
\caption{Q: ``How many objects have shape of orange object?''}
\end{subfigure}
\hfill
\begin{subfigure}{0.45\textwidth}
\includegraphics[width=.48\linewidth, frame]{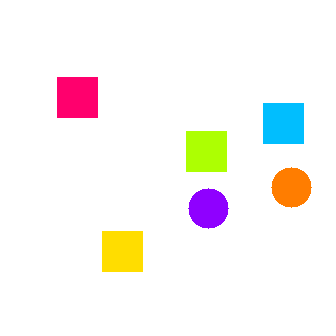}
\hfill
\includegraphics[width=.48\linewidth, frame]{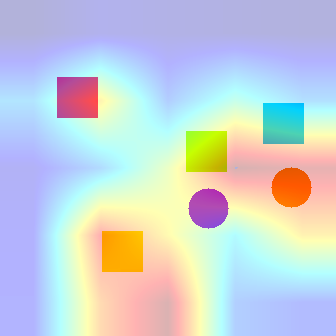}
\caption{Q: ``Which shape is closest to the yellow square?''}
\end{subfigure}
\vfill
\medskip
\begin{subfigure}{0.45\textwidth}
\includegraphics[width=.48\linewidth, frame]{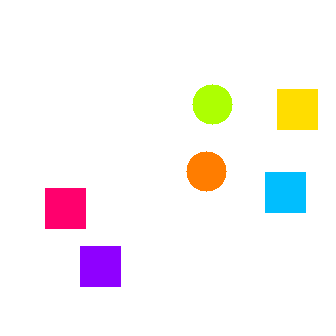}
\hfill
\includegraphics[width=.48\linewidth, frame]{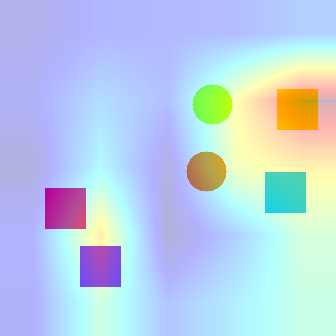}
\caption{Q: ``Which shape is closest to the purple square?''}
\end{subfigure}
\hfill
\begin{subfigure}{0.45\textwidth}
\includegraphics[width=.48\linewidth, frame]{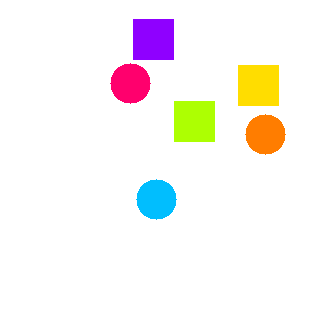}
\hfill
\includegraphics[width=.48\linewidth, frame]{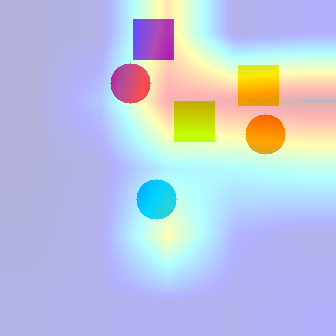}
\caption{Q: ``How many objects have shape of pink object?''}
\end{subfigure}
\vfill
\medskip
\begin{subfigure}{0.45\textwidth}
\includegraphics[width=.48\linewidth, frame]{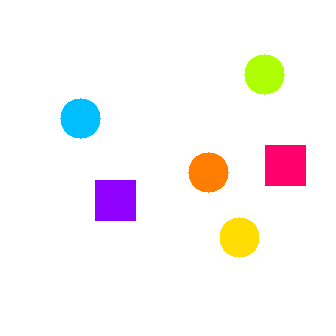}
\hfill
\includegraphics[width=.48\linewidth, frame]{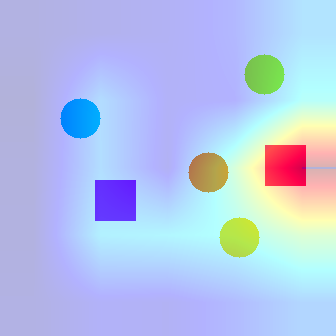}
\caption{Q: ``Which shape is furthest from the green circle?''}
\end{subfigure}
\hfill
\begin{subfigure}{0.45\textwidth}
\includegraphics[width=.48\linewidth, frame]{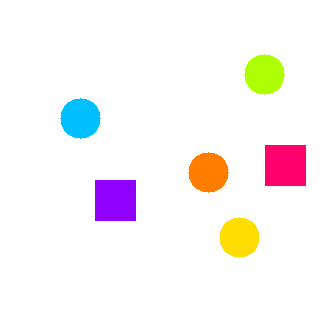}
\hfill
\includegraphics[width=.48\linewidth, frame]{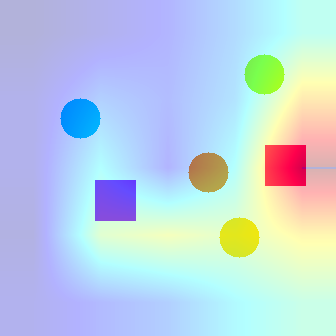}
\caption{Q: ``Which shape is closest to the blue circle?''}
\end{subfigure}
\vfill
\medskip
\begin{subfigure}{0.45\textwidth}
\includegraphics[width=.48\linewidth, frame]{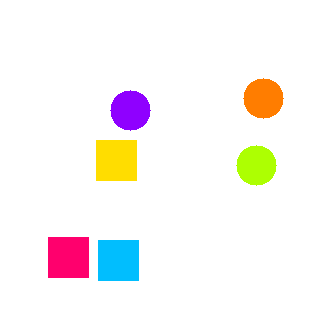}
\hfill
\includegraphics[width=.48\linewidth, frame]{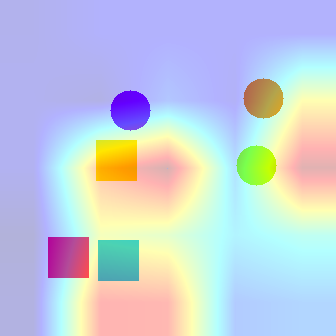}
\caption{Q: ``Which shape is furthest from the green circle?''}
\end{subfigure}
\hfill
\begin{subfigure}{0.45\textwidth}
\includegraphics[width=.48\linewidth, frame]{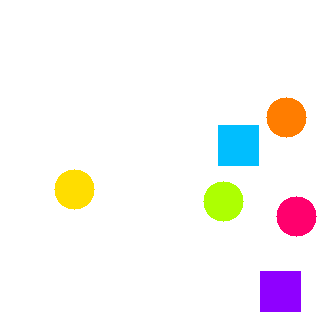}
\hfill
\includegraphics[width=.48\linewidth, frame]{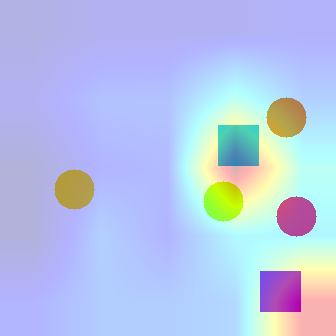}
\caption{Q: ``Which shape is furthest from the pink circle?''}
\end{subfigure}
\end{subfigure}
\caption{Grad-CAM's salience heatmaps per image/question. \textit{Left}, the raw image. \textit{Right}, Grad-CAM's explanatory heatmap.}
\label{gradcam}
\end{figure*}

Grad-CAM is a first-order explanatory tool that ranks different image regions and produces a heatmap of saliences. As shown qualitatively in Figure \ref{gradcam}, Grad-CAM's heatmaps are much harder to interpret and to reverse engineer questions and objects from compared to the results obtained from Taylor-CAM, shown in Figure 3 of the main paper, as corroborated quantitatively with our human study. 

\section{Interactional Relation Network (IRN)}

A standard RN pools a set of feature vectors $O = \{o_1, ..., o_n\}$, their corresponding positional encodings $C = \{c_1, ..., c_n\}$, and a question $q$ as follows:

\begin{equation}
\text{RN}(O, C, q) = f_{\phi}\Bigg(\sum\limits_{i, j} g_{\theta}(o_i, o_j, c_i, c_j, q)\Bigg),
\end{equation}

where $f$ and $g$ are modeled by neural networks parameterized by $\phi$ and $\theta$ respectively.

We observed through Taylor-CAM that many of the top interactions in the RN's reasoning were between individual regions and themselves, even when we zeroed out diagonals, such as in Figure \ref{selfinteracting} of the main paper. We found that we could mitigate this by making a simple modification to the RN architecture which we found to yield better test accuracy:

\begin{multline}
\text{IRN}(O, C, q) =\\ f_{\phi}\Bigg(\sum\limits_{i, j} g_{\theta}(h_{\psi}(o_i, c_i, q), h_{\psi}(o_j, c_j, q), c_i, c_j, q)\Bigg),
\end{multline}

where $h$ is an MLP parameterized by $\psi$.

We refer to this as Interactional Relation Network (IRN) since it explicitly separates within its architecture the concerns of reasoning about interactions from reasoning about individual objects. While IRN does not relate to Taylor-CAM directly, it does highlight how visualizing a network's relational reasoning can inspire potential ideas for improvement to a network's architecture. 

\section{Taylor-CAM Pipeline}

In Figure \ref{pipeline}, we illustrate the full pipeline of Taylor-CAM. A model consisting of a CNN and Relation Network predicts answers based on images and questions. Taylor-CAM intercepts the model's gradients, reverse engineers the question asked, and visualizes for a human observer. 

Given three possible question categories (\textit{closest}, \textit{furthest}, and \textit{same shape}), the user is able to interpret which the model is reasoning about by looking at Taylor-CAM's proposed interactions, as shown quantitatively in Table \ref{questions} of the main paper.

\section{Biomedical Analysis}

\begin{figure*}
  \centering
\begin{minipage}{\textwidth}
\centering
\begin{subfigure}{0.48\textwidth}
\includegraphics[width=\linewidth]{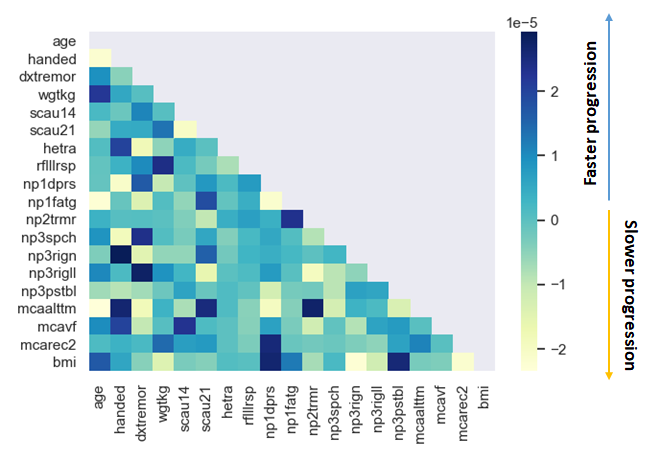}
  \caption{Top $2$-way interactions for MoCA fast progression} 
\end{subfigure}
\hfill
\begin{subfigure}{0.48\textwidth}
\begin{center}
        \setlength{\tabcolsep}{2pt}
        \resizebox{\textwidth}{!}{
        \begin{tabular}{lc}
        \hline $N$-Way Interaction & Strength \\
        \hline id\_num, scau20, mcarec4 & 4.77E-06 \\id\_num, drmagrac, mcarec4 & 4.69E-06 \\educyrs, np1apat, bmi & 4.56E-06 \\np1dprs, np2walk, np3pstbl & 4.27E-06 \\
        \hline scau13, np1slpn, np1cnst, nhy & 6.00E-07 \\scau11, scau13, scau20, bmi & 5.66E-07 \\scau11, scau13, np1slpn, nhy & 5.64E-07\\ scau13, scau20, np1urin, nhy & 5.43E-07 \\
        \hline slplmbmv, np1dprs, np2walk, np3rigru, np3pstbl & 1.23E-07 \\slplmbmv, np1dprs, np2walk, np3rign, np3pstbl & 1.22E-07 \\scau5, np1dprs, np2walk, np3rigru, np3pstbl & 1.19E-07 \\slplmbmv, np1dprs, np2walk, np3pstbl, mcarec2 & 1.18E-07 \\
        \hline
        \end{tabular}
        }
        \end{center}
  \caption{Top $3$-way, $4$-way, and $5$-way interactions for MoCA fast progression} 
  \vfill
\end{subfigure}
    \end{minipage}
    \caption{Interaction effects for classifying fast clinical progression of MoCA scores from baseline}
    \label{bio}
\end{figure*}

\begin{figure*}
  \centering
\begin{minipage}{\textwidth}
\centering
\begin{subfigure}{0.48\textwidth}
\includegraphics[width=\linewidth]{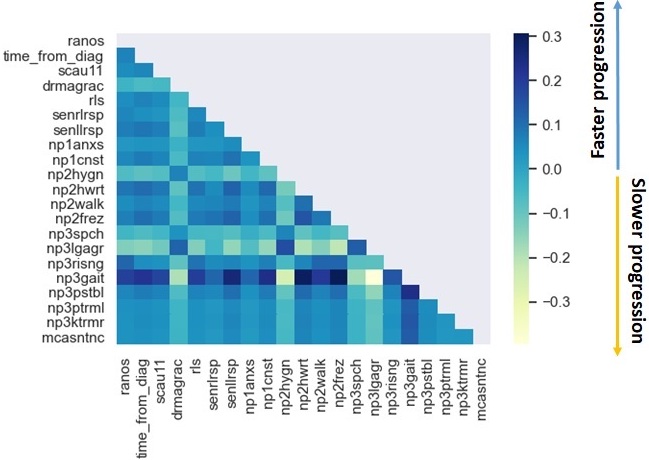}
   \caption{Top $2$-way interactions for uMCA fast progression} 
\end{subfigure}
\hfill
\begin{subfigure}{0.48\textwidth}
\begin{center}
        \setlength{\tabcolsep}{2pt}
        \resizebox{\textwidth}{!}{
        \begin{tabular}{lc}
        \hline $N$-Way Interaction & Strength \\
        \hline scau1, np3lgagr, np3risng & 1.97E+00 \\scau1, scau9, np3lgagr & 1.41E+00 \\scau1, np2hwrt, np3lgagr & 1.31E+00 \\scau1, senllrsp, np3lgagr & 1.19E+00 \\\hline time\_from\_diag, scau1, np2frez, np3lgagr & 5.39E+00 \\scau1, np2walk, np2frez, np3lgagr & 5.28E+00 \\ranos, scau1, np2frez, np3lgagr & 5.25E+00 \\scau1, rls, np2frez, np3lgagr & 5.21E+00 \\\hline scau1, np3lgagr, np3risng, np3gait, np3rtarl & 1.93E+01 \\scau1, np2hygn, np3lgagr, np3risng, np3gait & 1.80E+01 \\scau1, mslarsp, np3lgagr, np3risng, np3gait & 1.68E+01 \\dxrigid, scau1, np3lgagr, np3risng, np3gait & 1.47E+01 \\\hline
        \end{tabular}
        }
        \end{center}
\caption{Top $3$-way, $4$-way, and $5$-way interactions for uMCA fast progression} 
\end{subfigure}
    \end{minipage}
    \caption{Interaction effects for classifying fast clinical progression of uMCA scores from baseline}
    \label{bio2}
\end{figure*}

We applied these techniques to the Parkinson’s Progression Marker Initiative (PPMI) study (\url{http://www.ppmi-info.org/}) dataset, which follows persons living with early-stage Parkinson’s disease for up to approximately eight years collecting clinical and biological data from participants. Parkinson’s disease (PD) is a neurodegenerative progressive disease, characterized clinically by motor (\textit{e.g.}, tremor, rigidity) and non-motor (\textit{e.g.}, cognition and autonomic dysfunction) symptoms that vary over time within and between patients. Progression of motor and non-motor symptoms are likely not independent of each other. Instead, collateral damage may be inflicted multilaterally with non-motor and motor pathological features progressing interdependently. As an example, depressive symptoms in Parkinson’s disease are common and may perpetuate motor and cognitive deficits, which could impact function, and ultimately diminish quality of life. Therefore, it is necessary to take as comprehensive of an approach as possible in unraveling the clinical progression of Parkinson’s disease. 

As PD progresses, cognitive impairment leading to dementia may affect up to 80\% of patients, ultimately impairing one’s functional independence. Within the PPMI study, we tested $2$-, $3$-, $4$- and $5$-way interactions to understand multivariable features at baseline that distinguish patients with a more severe progression in decline of cognitive function (“fast progressors”) compared to those with a more benign course of cognitive changes, as measured by the MontrealCognitive Assessment (MoCA) scale. Top $2$-way interaction effects identified (Figure \ref{bio}) among “fast progressors” included feature interactions between handedness (handed) and severity of rigidity in the neck (np3rign); presence of resting tremor at disease diagnosis (dxtremor) and severity of rigidity in the lower extremities (np3rigll); and, severity of tremor (np2trmr) and alternating trail making test from the MoCA scale (mcaalttm) -- which ultimately is a measure of processing speed, mental flexibility, ability to sequence, and visuo-motor skills. Each of these features individually have some established associations with cognitive dysfunction or neuropsychological disorders; however, their interactions together have not been previously considered. For example, handedness, has been significantly associated with functional connectivity between language networks, as well as specific genetic loci implicated in the pathogenesis of neurologic disorders including Parkinson’s disease \cite{bio1}. More severe rigidity symptoms in Parkinson’s disease are also associated with faster cognitive decline \cite{bio2}.  Our analysis, for the first time, suggests that measures of both handedness and rigidity severity together are important to consider when predicting faster cognitive progression in Parkinson’s disease. As shown in Figure \ref{bio}, we provide $3$-, $4$-, and $5$-way interactions between features.

When broadening the interactions to 3-, 4-, and 5-way interactions between features that predict fast cognitive decline, we observe additional features with some consistency (Figure \ref{bio2} provides 3-, 4-, and 5-way interactions between features). Broadly speaking, some of the most important interactions occurred between symptoms of autonomic dysfunction: urinary (np1urin, scau11, scau13) and constipation issues (np1cnst, scau5), problems tolerating cold/heat [scau20]); mood and sleep disturbances: depression (np1dprs), apathy (np1apat), and restless sleep (np1slpn, slplmbmv); postural instability and balance issues (np2walk, np3pstbl); overall severity of Parkinson’s disease (nhy); and, memory impairment: delayed recall (mcarec2, mcarec4). Each of these symptoms, singularly, have been thought to be associated with cognitive impairment \cite{bio3, bio4, bio5, bio6}. It is novel, yet biologically plausible to consider these symptoms interacting, as the neuropathology underlying Parkinson’s disease involves multiple areas of the brain and nervous system beyond the nigrostriatal dopamine pathway. For instance, Lewy Body pathology affects the limbic cortex and frontal neocortical areas, sympathetic ganglia and even the peripheral autonomic nervous system including the myenteric plexus \cite{bio7}. 

We also performed our analysis to predict fast progression of ambulatory impairment which stems from worsening progression of motor symptoms and is a major source of disability for patients with Parkinson’s disease. Severity of ambulatory impairment was measured by an ambulatory capacity score derived from sum of scores of the MDS-UPDRS items 2.13 (freezing), 2.12 (walking and balance), 3.10 (gait), 3.12 (postural stability), and 3.11 (freezing of gait). The top 2-way interaction among “fast progressors” of ambulatory capacity was between severity of freezing (np2frez) and gait (np3gait), which is unsurprising as both are components of the ambulatory capacity scale score. Interestingly, however, the next top 2-way interactions were between handwriting (np2hwrt) and gait (np3gait); and, sensory of legs (senllrsp) and gait (np3gait).  Worsening of handwriting is often reported as an initial symptom of Parkinson’s disease and is reported to be more problematic in people with Parkinson’s disease who experience freezing of gait \cite{bio8, bio9}. Periphery sensory defects in the lower limbs are also commonly noted in people with Parkinson’s disease, and could be a main contributor to balance control issues and postural instability \cite{bio10}. As interactions were expanded to 3-, 4-, and 5-ways, other items that were consistently identified were difficulty in swallowing and chewing (scau1), and leg agility (np3lgagr). While these items are not often considered as obvious predictors of ambulatory capacity by themselves, their interactions with some more apparent features (e.g., gait, freezing, arising from chair) provide new insights on how symptoms in Parkinson’s disease patients contribute to disease progression. 


\end{document}